\documentclass[usenames,dvipsnames]{article}

\usepackage[margin=2.2cm]{geometry}

\usepackage{authblk}

\usepackage[utf8]{inputenc}
\usepackage{glossaries} 
\usepackage[colorlinks=true,allcolors=blue]{hyperref}
\makeglossaries
\usepackage[french, english]{babel} 

\usepackage[T1]{fontenc} 
\usepackage{xspace}

\usepackage{natbib} 
\usepackage{graphicx}
\usepackage[dvipsnames]{xcolor} 
\graphicspath{{./}}

\usepackage{algorithm}
\usepackage{algorithmic}

\usepackage{amsmath} 
\usepackage{mathrsfs} 
\usepackage{amsthm} 
\usepackage{amsfonts} 
\usepackage{amssymb} 
\usepackage{dsfont} 
\usepackage{stmaryrd} 
\usepackage[capitalize, noabbrev]{cleveref} 

\usepackage[colorinlistoftodos,textwidth=2.3cm]{todonotes}

\newcommand \ie {\emph{i.e.}\xspace}

\newtheorem{theorem}{Theorem}
\newtheorem{lemma}[theorem]{Lemma}
\newtheorem{proposition}[theorem]{Proposition}

\newtheorem{remark}{Remark}[section]

\usepackage{enumitem}

\makeatletter
\newcommand\hyp[1]{\item[\textbf{(#1)}]
  \edef\@currentlabel{(#1)}}
\makeatother
\makeatletter
\newcommand\varitem[1]{\item[\textbf{(A\arabic{enumi}{$#1$})}]
 \edef\@currentlabel{(A\arabic{enumi}{$#1$})}}
\makeatother

\newcommand{\R}{\mathbb{R}\xspace}

\newcommand{\dd}{\, \mathrm d}
\renewcommand \H {\mathcal{H}}
\newcommand \F {\mathcal{F}}

\newcommand \llb {\llbracket}
\newcommand \rrb {\rrbracket}

\newcommand \ind {\mathds{1}} 
 
\newcommand \indic[1] {\mathbf{1}_{#1}} 
\newcommand{\iH}[2]{\left\langle #1, #2 \right\rangle_{\H}}
\newcommand \e {\operatorname{e}}
\newcommand{\erf}{\operatorname{erf}}

\newcommand{\prob}[1]{\mathbb{P} \left( #1 \right)}

\newcommand{\given}{\mid}

\newcommand{\argmin}{\operatorname{arg \, min}}

\newcommand \minimize[1] {\underset{#1}{\operatorname{minimize}}~}
\newcommand \maximize[1] {\underset{#1}{\operatorname{maximize}}~}
\newcommand \st {\operatorname{s.t. }}

\newenvironment{disarray}
 {\everymath{\displaystyle\everymath{}}\array}
 {\endarray}
\newenvironment{opb*} 
{
	\[
		\begin{disarray}{c@{\hspace*{0.05cm}}l}
}
{
		\end{disarray}
	\]\par
}
\newcounter{opb}
\renewcommand*{\theopb}{(P\arabic{opb})}
\crefname{opb}{Problem}{Problems}

\newenvironment{opb}[1]
{
	\stepcounter{opb}
	\equation
	  \label[opb]{#1}
		\begin{disarray}{c@{\hspace*{0.05cm}}l}
}
{
		\end{disarray}
		\tag*{\theopb}
	\endequation
}
\newenvironment{starray} 
{
	\left\{ \begin{disarray}{l}
}
{
	\end{disarray} \right.
}
\newenvironment{starrayd} 
{
	\left\{ \begin{disarray}{l@{\quad}l}
}
{
	\end{disarray} \right.
}

\newcommand{\ms}[1]{\todo[linecolor=Aquamarine, bordercolor=Aquamarine, backgroundcolor=white, size=\footnotesize]{#1}} 

\newcommand{\myacronym}[4] 
{
	\newglossaryentry{#1}
	{
		type=\acronymtype,
		name={#2},
		description={#3},
		text={#2},
		first={#3 (#2)},
		firstplural={#4 (#2s)},
		short={#2}
	}
	\expandafter\newcommand\csname #1\endcsname{\gls{#1}\xspace} 
	\expandafter\newcommand\csname #1s\endcsname{\glspl{#1}\xspace} 
}

\myacronym{rkhs}{RKHS}{reproducing kernel Hilbert space}{reproducing kernel Hilbert spaces}
\myacronym{ovk}{OVK}{operator-valued kernel}{operator-valued kernels}
\myacronym{svm}{SVM}{support vector machine}{support vector machines}
\myacronym{svc}{SVC}{support vector classification}{-}
\myacronym{svr}{SVR}{support vector regression}{-}
\myacronym{krr}{KRR}{kernel ridge regression}{-}
\myacronym{mmr}{MMR}{maximum margin robot}{maximum margin robots}

\myacronym{lda}{LDA}{linear discriminant analysis}{Linear discriminant analysis}
\myacronym{qda}{QDA}{quadratic discriminant analysis}{Quadratic discriminant analysis}
\myacronym{samme}{SAMME}{Stagewise Additive Modeling using a Multiclass Exponential loss}{-}
\myacronym{dbscan}{DBSCAN}{density-based spatial clustering of applications with noise}{-}
\myacronym{bic}{BIC}{Bayesian information criterion}{Bayesian information criteria}
\myacronym{aic}{AIC}{Akaike information criterion}{Akaike information criteria}
\myacronym{pca}{PCA}{principal component analysis}{-}
\myacronym{svd}{SVD}{singular value decomposition}{-}
\myacronym{rip}{RIP}{restricted isometry property}{-}
\myacronym{mds}{MDS}{multidimensional scaling}{-}
\myacronym{smacof}{SMACOF}{scaling by majorizing a complicated function}{-}
\myacronym{lstm}{LSTM}{Long Short-Term Memory}{-}

\myacronym{mst}{MST}{minimum spanning tree}{minimum spanning trees}

\myacronym{admm}{ADMM}{alternating direction method of multipliers}{-}
\myacronym{smo}{SMO}{sequential minimal optimization}{-}
\myacronym{kkt}{KKT}{Karush-Kuhn-Tucker}{-}
\myacronym{pdcd}{PDCD}{Primal-dual coordinate descent}{-}

\myacronym{rv}{RV}{random variable}{random variables}
\myacronym{qr}{QR}{quantile regression}{-}
\myacronym{iid}{\emph{iid}}{independent and identically distributed}{-}
\myacronym{cadlag}{càdlàg}{right continuous with left limits}{-}
\myacronym{pdf}{pdf}{probability density function}{probability density functions}
\myacronym{cdf}{cdf}{cumulative distribution function}{cumulative distribution functions}
\myacronym{mle}{MLE}{maximum likelihood estimator}{maximum likelihood estimators}
\myacronym{emalg}{EM}{expectation-maximization algorithm}{-}

\myacronym{psd}{PSD}{positive semi-definite}{-}
\myacronym{pd}{PD}{positive definite}{-}

\renewcommand{\ms}[1]{}

\title{
  Nonparametric estimation of Hawkes processes with RKHSs
  }
\author[1]{Anna Bonnet}
\author[1]{Maxime Sangnier}
\affil[1]{Sorbonne Université and Université Paris Cité, CNRS, Laboratoire de Probabilités, Statistique et Modélisation, F-75005 Paris, France}

\begin{document}
	\maketitle

	\begin{abstract}

This paper addresses nonparametric estimation of nonlinear multivariate Hawkes processes,
where the interaction functions are assumed to lie in a reproducing kernel Hilbert space (RKHS).
Motivated by applications in neuroscience, the model allows complex interaction functions,
in order to express exciting and inhibiting effects, but also a combination of both (which
is particularly interesting to model the refractory period of neurons), and considers in return
that conditional intensities are rectified by the ReLU function.
The latter feature incurs several methodological challenges, for which workarounds are proposed in this paper.
In particular, it is shown that a representer theorem can be obtained for approximated versions of the log-likelihood and the least-squares criteria.
Based on it, we propose an estimation method, that relies on two
common
approximations
(of the ReLU function and of the integral operator).
We provide a bound
that controls the impact of these approximations.
Numerical results on synthetic data confirm this fact as well as the good asymptotic behavior of the proposed estimator.
It also shows that our method achieves a better performance compared to related nonparametric estimation techniques
and suits neuronal applications.


	\end{abstract}

	\section{Introduction}
Hawkes processes are a class of past-dependent point processes, widely used in many applications such as seismology \citep{ogata_statistical_1988}, criminology \citep{olinde_selflimiting_2020} and neuroscience \citep{reynaud-bouret_inference_2013} for their ability to capture complex dependence structures. In their multidimensional version \citep{ogata_statistical_1988}, Hawkes processes can model pairwise interactions between different types of events, allowing to recover a connectivity graph between different features.  Originally developed by \cite{hawkes_spectra_1971} in order to model self-exciting phenomena, where each event increases the probability of a new event occurring, many extensions have been proposed ever since. In particular, nonlinear Hawkes processes have been introduced notably to detect inhibiting interactions, when an event can decrease the probability of another one appearing.
Hawkes processes with inhibition are notoriously more complicated to handle due to the loss of many properties of linear Hawkes processes such as the cluster representation and the branching structure of the process \citep{hawkes_cluster_1974}.

Since the first article on nonlinear Hawkes processes \citep{bremaud_stability_1996} proving in particular their existence, many works have focused on inhibition in the past few years. Among them, limit theorems have been established in \citep{costa_renewal_2020} while \cite{duval_interacting_2022} obtained mean-field results on the behaviour of two neuronal populations. Regarding statistical inference, in the frequentist setting we can mention the exact maximum likelihood procedure of \cite{bonnet_inference_2023}, the least-squares approach by \cite{bacry_sparse_2020} and the nonparametric approach based on
Bernstein-type polynomials
by \cite{lemonnier_nonparametric_2014}. While the first one proposes an exact inference procedure, it is restricted to exponential kernels. The other two methods do not make such an assumption but they consider a strong approximation which requires the intensity function to remain almost always positive: this can be inacurrate when the inhibiting effects are strong, providing estimation errors in such settings, as empirically shown in \citep{bonnet_inference_2023}.
In the Bayesian framework, \cite{sulem_bayesian_2024} proposed a nonparametric estimation procedure for kernel functions with
bounded support. The authors then developed a variational inference procedure \citep{sulem_scalable_2023} in order to reduce the computational cost of their method. Finally, \cite{deutsch_estimating_2022} investigated a parametric inference method based on a new reparametrisation of the process.

Motivated by applications in neuroscience, we aim at developing a flexible model to capture complex neuronal interactions from spike train data. The current knowledge suggests several specificities regarding neuronal interactions that require appropriate modeling. Firstly, several works imply indeed that neuronal interactions can describe both inhibiting and exciting effects \citep{berg_synaptic_2013,bonnet_inference_2023}.
Moreover, neurons also exhibit a refractory period \citep{lovelace_chapter_1994}, that is a recovery period following a spike during which the neuron cannot emit another spike.
In the end, it is likely that a neuron has a short-term self-inhibiting effect due to its refractory period combined with a long-term self-exciting behaviour,
as illustrated by the synthetic interaction functions (or triggering kernels) in \cref{img:intro} (top left, blue curve).
This kind of interaction is very challenging to model with classical kernel functions, such as exponential ones \citep{bonnet_inference_2023}.

In order to account for the specificities of neuronal activity, we propose a nonparametric estimation method to model and estimate complex interaction functions, that can in particular change signs along time (see the green curve in \cref{img:intro}).
The flexibility of the approach relies on the mild assumption that the interaction functions belong to a \rkhs.
Our method is supported by theoretical guarantees, including representer theorems and approximation bounds.
An
empirical study highlights the performance of our approach compared to alternative ones
including the exponential model \citep{bonnet_inference_2023} and Bernstein-type polynomial approximation \citep{lemonnier_nonparametric_2014}.

\begin{figure}[ht]
  \center
  \includegraphics[width=.5\textwidth]{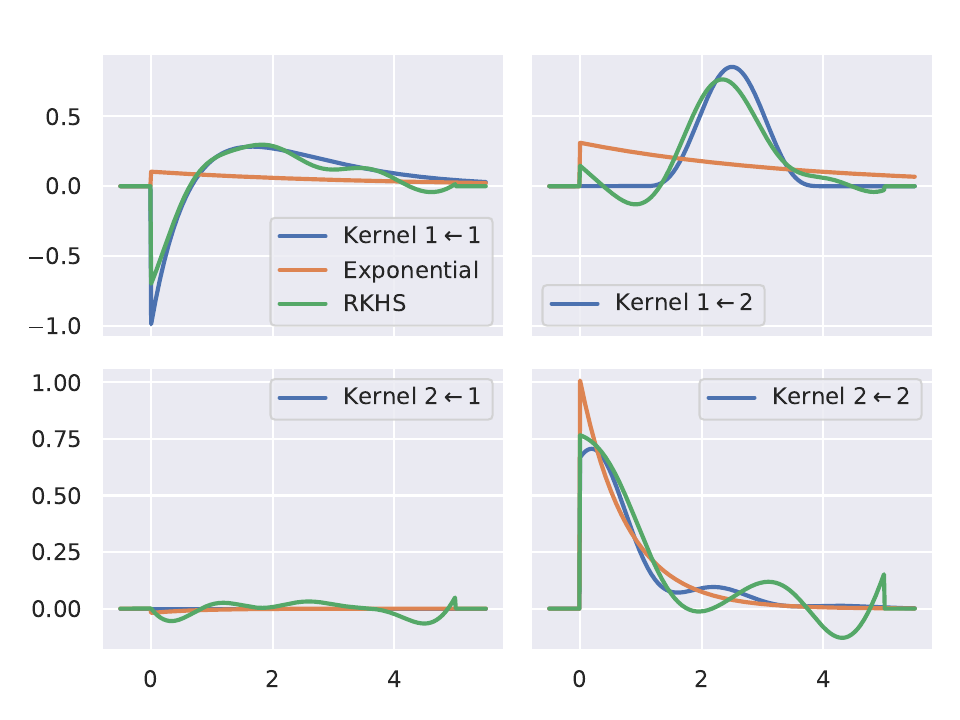}
  \caption{Example of estimation of the triggering kernels (blue) of a Hawkes process with the exponential model (orange) and the proposed method (green).}
  \label{img:intro}
\end{figure}

	\section{Related work}
Nonparametric estimation of point processes, and particularly of Hawkes processes, has been largely studied in the literature, and is still an active field of research.
The earliest focus has been made on the Bayesian inference for Poisson processes.
This one is often based on the Gaussian-Cox model, which is an inhomogeneous Poisson process with a stochastic intensity function modulated by a Gaussian process (then rectified by a link function in order to guarantee non-negativity of the intensity).
Examples include \citep{moller_log_1998,cunningham_fast_2008,mohler_selfexciting_2011,linderman_scalable_2015} (which both require domain gridding),
\citep{linderman_discovering_2014} for a combination of exciting point processes with random graph models,
and \citep{adams_tractable_2009,samo_scalable_2015} for inference methods based on tractable Markov Chain Monte Carlo algorithms.
Concurrently, a variational inference schemes were proposed \citep{lloyd_variational_2015,lian_multitask_2015}, which scale better than the method introduced by \citet{samo_scalable_2015}.

Efficient nonparametric Bayesian estimation of linear Hawkes processes is well exemplified in the work by \citet{zhang_efficient_2019}, which is based on sampling random branching structures with a Gibbs procedure.
Since variational inference enjoys faster convergence than Gibbs sampling, the method proposed by
\citet{lloyd_variational_2015} has then been extended to inference of Hawkes processes \citep{zhang_variational_2020}, where interactions functions are square transformations of Gaussian processes, then to model a nonparametric baseline intensity in addition to the triggering kernels \citep{zhou_efficient_2021},
while still enjoying a scalable feature to suit large real world data.
Other works modeling both the baseline intensity and the triggering kernels as transformations of Gaussian processes exist, for instance related to mean-field variational algorithms \citep{zhou_scalable_2019,zhou_efficient_2020}.

A different line of works regarding Bayesian inference of Hawkes processes is modeling interaction kernels as piecewise constant functions and setting priors on the function levels \citep{donnet_nonparametric_2020,browning_flexible_2022,sulem_scalable_2023,sulem_bayesian_2024}.
The algorithmic part relies on a Markov chain Monte Carlo estimator of the posterior distribution, close to \citep{zhang_efficient_2019}.

Nonparametric frequentist estimation methods are not left out:
a famous approach relies on relating the interaction functions with the second order statistics of its counting process, leading to an estimation method based on the solution to a Wiener-Hopf equation \citep{bacry_nonparametric_2012,bacry_first_2016}.
Neural networks were also used to estimate general point processes \citep{shchur_intensityfree_2019} (see references therein) and Hawkes processes \citep{pan_selfadaptable_2021}, requiring a sufficient amount of data and of computational resources.
Models are based on \lstm recurrent neural networks \citep{du_recurrent_2016,mei_neural_2017},
independent neural networks for each triggering kernel \citep{joseph_shallow_2022,joseph_neural_2024,joseph_nonparametric_2024}
(which is more expressive than \lstms since it allows to recover the components of the Hawkes process and in particular its Granger causality graph),
and transformers \citep{zuo_transformer_2020,meng_interpretable_2024,wu_learning_2024}.

Nonparametric inference methods also include basis decomposition or linear combination of pre-defined functions such as
piece-wise constant functions \citep{marsan_extending_2008,reynaud-bouret_adaptive_2010,hansen_lasso_2015,eichler_graphical_2017},
exponential functions coming from Bernstein-type polynomials \citep{lemonnier_nonparametric_2014},
smoothing splines \citet{lewis_nonparametric_2011,zhou_learning_2013} (estimated by solving Euler-Lagrange equations),
Gaussian functions \citep{xu_learning_2016},
and cosine series \citep{chen_nonparametric_2016}.
The bulk of these methods are designed for linear Hawkes processes and rely either on the expectation-maximization technique introduced by \citet{lewis_nonparametric_2011} or on standard numerical optimization algorithms.
Some other approaches are noticeable, such as those related to kernel estimators \citep{chen_nonparametric_2016} and autoregressive processes \citep{kirchner_estimation_2017,kirchner_nonparametric_2018,kurisu_discretization_2017}.

Despite the vast literature regarding nonparametric estimation of point processes, only a few works addressed \rkhss.
\citet{flaxman_poisson_2017,flaxman_poisson_2017a} proposed an inference method for inhomogeneous Poisson processes,
while online estimation of linear Hawkes processes with time discretization is addressed in \citep{yang_online_2017}.
It is important to note that online estimation is not aimed at minimizing the estimation criterion (for instance the negative log-likelihood or the least-squares contrast) but rather at minimizing its regret.
Thus, this approach is, from an optimization point of view, suboptimal when the data is available offline.

Up to our knowledge, no work tackles inference of triggering kernels from an \rkhs in a batch setting (while in many applications, in particular in neuroscience, the data is available offline rather than online) and even less in inhibition setting.
In addition, nonparametric frequentist methods able to handle inhibition are either based on Bernstein-type polynomials \citep{lemonnier_nonparametric_2014} or on neural networks \citep{mei_neural_2017,meng_interpretable_2024,joseph_neural_2024}.
As a result, our contribution lies in proposing an offline learning procedure of nonlinear Hawkes processes, for which triggering kernels come from an \rkhs.
The method rests on representer theorems and approximation bounds, proved in \cref{app:proof1,app:proof2,app:proof3,app:approximation}.  
The proposed estimation method is implemented in Python and will be freely available on GitHub.\footnote{\url{https://github.com/msangnier/kernelhawkes}}


	\section{Inference of Hawkes processes}
Let \(N = (N^{(1)}, \dots, N^{(d)})\) be a multivariate point process on \(\R_+\),
defined by its conditional intensities \(\lambda^{(1)}, \dots, \lambda^{(d)}\):
\[
  \forall j \in \llb 1, d \rrb,
  \forall t \in \R_+:
  \quad
  \lambda^{(j)}(t) = \lim_{\Delta t \to 0} \frac{\prob{N^{(j)}([t, t+\Delta t)) = 1 \given \F_t}}{\Delta t},
\]
where \(\F_t = \sigma \left( N^{(j)} \left([0, s) \right), s \le t, j \in \llb 1, d \rrb \right)\) is the internal history of \(N\) \citep{bremaud_stability_1996}.
It is assumed that \(N\) is a nonlinear Hawkes process characterized by:
\[
  \forall j \in \llb 1, d \rrb,
  \forall t \in \R_+:
  \quad
  \lambda^{(j)}(t) = \varphi \left( \mu_j + \sum_{\ell=1}^d \int_{[0, t)} g_{j \ell}(t - s) \, N_\ell(\mathrm d s) \right),
\]
where \(\varphi : \R \to \R_+\) is a non-negative link function,
\(\mu_1, \dots, \mu_d > 0\) are baseline intensities,
and \(g_{1 1}, g_{1 2} \dots, g_{d d}\) are interaction functions from \(\R\) to \(\R\).

This work is aimed at designing an inference method for \(N\) based on \rkhss.
Concretely, let \(\H\) be a prescribed \rkhs of reproducing kernel \(k : \R^2 \to \R\) \citep{berlinet_reproducing_2004}
and \(A > 0\) some arbitrary bound.
From now on, it is assumed that for all \((j, \ell) \in \llb 1, d \rrb^2\), \(g_{j \ell} = \left( h_{j \ell} + b_{j \ell} \right) \indic{[0, A]}\) (with \(\indic{[0, A]}\) being the indicator function of \([0, A]\)),
where \(h_{j \ell} \in \H\) is the functional part of the interaction function \(g_{j \ell}\)
and \(b_{j \ell} \in \R\) its offset.
As commonly done (see for instance \citep{staerman_fadin_2023}), it is assumed that events cannot have an influence far in the future, which translates here by the interaction function \(g_{j \ell}\) having a bounded support.
With a slight abuse of notation, we now note \(\theta = \left( (\mu_j)_{1 \le j \le d}, (h_{j \ell})_{1 \le j, \ell \le d}, (b_{j \ell})_{1 \le j, \ell \le d} \right)\) the parameter to be estimated and
\(\Theta^+ = \R_+^d \times \H^{d^2} \times \R^{d^2}\)
its corresponding set
(we allow baseline to cancel for optimization purposes).

Now, let, for all \(j \in \llb 1, d \rrb\), \(\left( T_n^{(j)} \right)_{n \ge 1}\) be
a sorted realization of \(N^{(j)}\) in \(\R_+\),
and \(N^{(j)}_t = N^{(j)} \left([0, t)\right)\) the number of these times in the interval \([0, t)\) (for any \(t \ge 0\)).
With this notation, conditional intensities
read:
\[
  \forall j \in \llb 1, d \rrb,
  \forall t \in \R_+:
  \quad
  \lambda^{(j)}(t) = \varphi \left( \mu_j + \sum_{\ell=1}^d \sum_{i = 1}^{N^{(\ell)}_t} g_{j \ell} \left(t - T_i^{(\ell)} \right) \right),
\]
which will be left-continuous on event times here.
Assuming that times are observed in the interval \([0, T]\) (\(T>0\) is a fixed horizon),
there exist two common methods for estimating a Hawkes process.
The first one is by minimizing the negative log-likelihood \citep{ozaki_maximum_1979,daley_introduction_2003}:
\[
  \forall \theta \in \Theta^+,
  \quad
  L(\theta) = \sum_{j=1}^d \left[ \int_0^T \lambda^{(j)}(t) \dd t - \sum_{n=1}^{N^{(j)}_T} \log \left( \lambda^{(j)} \left( T_n^{(j)} \right) \right) \right],
\]
and the second is by minimizing an approximated least-squares contrast \citep{reynaud-bouret_goodnessoffit_2014,bacry_sparse_2020}:
\[
  \forall \theta \in \Theta^+,
  \quad
  J(\theta) = \sum_{j=1}^d \left[ \int_0^T \lambda^{(j)}(t)^2 \dd t - 2 \sum_{n=1}^{N^{(j)}_T} \lambda^{(j)} \left( T_n^{(j)}  \right) \right].
\]

Throughout this paper
and as it is common for kernelized estimation (in order to prevent overfitting the training data), a quadratic penalization is added to the contrast:
let \(\eta > 0\) be some regularization parameter,
the regularized estimation problem addressed in this paper being
\begin{opb}{opb:contrast}
  \minimize{\theta \in \Theta^+}
  & C(\theta)
  + \frac{\eta}{2} \sum_{1 \le j, \ell \le d} \|h_{j \ell}\|_\H^2,
\end{opb}
where \(C\) is either \(L\) or \(J\).

As highlighted in the forthcoming sections,
while the previous optimization problem will appear to be convex,
two pitfalls prevent the easy derivation of kernelized estimators of Hawkes processes.
The first one is the integral operator in \(L\) and \(J\),
and the second one is the nonlinear link function \(\varphi\).
The next sections discuss these problems and present workarounds.


	  \subsection{Linear Hawkes processes}
The linear Hawkes process is the original model of past-dependent processes \citep{hawkes_spectra_1971},
where interaction functions are supposed to have non-negative values: for all \((j, \ell) \in \llb 1, d \rrb^2\), \(g_{j \ell} : \R \to \R_+\),
which makes it possible to get rid of the link function \(\varphi\).
Then, conditional intensities are:
\[
  \forall j \in \llb 1, d \rrb,
  \forall t \in \R_+:
  \quad
  \lambda^{(j)}(t) = \mu_j + \sum_{\ell=1}^d \sum_{i = 1}^{N^{(\ell)}_t} \left( h_{j \ell} \left(t - T_i^{(\ell)} \right) + b_{j \ell} \right) \indic{[0, A]} \left(t - T_i^{(\ell)} \right).
\]
From a numerical point of view, it seems easier to estimate linear than nonlinear Hawkes processes because of losing the nonlinear function \(\varphi\).
This is partially true regarding the first difficulty (integrating conditional intensities) but comes with the price of non-negativity constraints on \(g_{j \ell}\), for all \((j, \ell) \in \llb 1, d \rrb^2\).
As a result, implementation of maximum likelihood and least-squares estimators is not trivial and only approximate derivations can be produced.
The first example is presented now, the second is a byproduct of \cref{prop:representer_thm_mle} (see the discussion below \cref{prop:representer_thm_mle}).

The first example consists in considering, for all \(j \in \llb 1, d \rrb\),
a Riemann approximation of the integral term:
\(\int_0^T \lambda^{(j)}(t)^2 \dd t \approx \frac{T}{M} \sum_{n=1}^M \lambda^{(j)}(\tau_n)^2\),
with \(\tau_n = \frac{n-1}{M} T\) (\(M\) being an integer greater than two),
and in discretizing, for all \(\ell \in \llb 1, d \rrb\), the non-negativity constraint
\(g_{j \ell} \ge 0\) to \(h_{j \ell}(x_n) + b_{j \ell} \ge 0\) for all \(x_n = \frac{n-1}{P-1} A\) (\(P\) being an integer greater than two).

\begin{proposition}[Representer theorem for discretized least-squares estimation of linear Hawkes processes]
  \label{prop:duality}
  Let \(\varphi\) be the identity function
  and
  \[
    \forall \theta \in \Theta^+,
    \quad
    J_M(\theta) = \sum_{j=1}^d \left[ \frac{T}{M} \sum_{n=1}^M \lambda^{(j)}(\tau_n)^2 - 2 \sum_{n=1}^{N^{(j)}_T} \lambda^{(j)} \left( T_n^{(j)}  \right) \right].
  \]
  Then, denoting \(\Theta = \R^d \times \H^{d^2} \times \R^{d^2}\),
  if the optimization problem
  \begin{opb*}
    \minimize{\theta \in \Theta}
    & J_M(\theta)
    + \frac{\eta}{2} \sum_{1 \le j, \ell \le d} \|h_{j \ell}\|_\H^2 \\
    \st & \begin{starray}
      \forall j \in \llb 1, d \rrb, \mu_j \ge 0\\
      \forall (j, \ell) \in \llb 1, d \rrb^2, \forall n \in \llb 1, P \rrb, h_{j \ell}(x_n) + b_{j \ell} \ge 0,
    \end{starray}
  \end{opb*}
  has a solution \(\theta\), it is of the form:
  \begin{align*}
    \forall (j, \ell) \in \llb 1, d \rrb^2,
    \quad
    h_{j \ell}
    &= \eta^{-1} \left[
      2 q_{j \ell}
      + \sum_{n=1}^P {\beta_n^{(j \ell)}} k \left( \cdot, x_n \right)
      - \sum_{n=1}^M {\alpha_n^{(j)}} r_{\ell n}
    \right],
  \end{align*}
  where \(({\alpha^{(j)}})_{1 \le j \le d} \in (\R^M)^d\), \(({\beta^{(j \ell)}})_{1 \le j, \ell \le d} \in (\R_+^P)^{d \times d}\) and for all \((j, \ell) \in \llb 1, d \rrb^2\):
  \[
    \begin{cases}
      q_{j \ell} = \sum_{1 \le n \le N_T^{(j)} \atop 1 \le i \le N_T^{(\ell)}} k \left(\cdot, T_n^{(j)} - T_i^{(\ell)} \right) \indic{0 < T_n^{(j)} - T_i^{(\ell)} \le A} \\
      r_{\ell n} = \sum_{i=1}^{N_T^{(\ell)}} k \left( \cdot, \tau_n - T_i^{(\ell)} \right) \indic{0 < \tau_n - T_i^{(\ell)} \le A}.
    \end{cases}
  \]
\end{proposition}

\cref{prop:duality} is usually called a representer theorem.
It states that even though the optimization problem considered is semi-parametric,
a solution of it is supported by the data, and can thus be described with an amount of parameters depending on the number of observations.
Here, \cref{prop:duality} tells that for discretized least-squares estimation of linear Hawkes processes,
a solution can be expressed with \(d \left( M + dP \right)\) parameters.

\cref{prop:duality} is rather a weak result in that it highly relies on discretization (embodied by arbitrarily large integers \(M\) and \(P\)), which entails high computational cost.
However, it has the merit of showing that nonparametric estimation is possible in an ideal situation.
In addition, let us remark that in the expected case where for all \((j, \ell) \in \llb 1, d \rrb^2\) and \(n \in \llb 1, P \rrb\), \(h_{j \ell}(x_n) + b_{j \ell} > 0\), then by complementary slackness of Karush-Kuhn-Tucker conditions, \(\beta_n^{(j \ell)} = 0\), leading to a much simpler form of \(h_{j \ell}\).

Regarding discretized maximum likelihood estimation, such a result remains unestablished, the logarithm function preventing us from isolating saddle points of the Lagrangian function.

	  \subsection{Nonlinear Hawkes processes}
As it is common in the literature (in order to be consistent with linear Hawkes processes),
we focus on nonlinear processes built on the ReLU link function \(\varphi : x \in \R \mapsto \max(0, x)\).
Moreover, we consider the convention \(\log(x) = -\infty\) for all \(x \le 0\).
Thereafter, we denote \(N_T = \sum_{j=1}^d N_T^{(j)}\) the total number of observed points from the process.

If it turns out that obtaining representer properties for nonlinear Hawkes process estimation based on \cref{opb:contrast} is quite intricate, \cref{prop:representer_thm_mle,prop:representer_thm_ls} present representer theorems for approximations of the objective functions.

\begin{proposition}[Representer theorem for approximated maximum likelihood estimation]
  \label{prop:representer_thm_mle}
  The negative log-likelihood \(L\) can be approximated by a function \(L_0 : \theta \in \Theta^+ \to \R\)
  such that if the approximated regularized maximum likelihood problem
  \[
    \argmin_{\theta \in \Theta^+} L_0(\theta) + \frac{\eta}{2} \sum_{1 \le j, \ell \le d} \|h_{j \ell}\|_\H^2
  \]
  admits a minimizer,
  it has a solution \(\theta\) of the form:
  \[
      \forall (j, \ell) \in \llb 1, d \rrb^2,
      \quad
      h_{j \ell} = \alpha^{(j \ell)}_0 r_\ell + \sum_{u=1}^{N_T^{(j)}} \alpha^{(j \ell)}_u q_{u j \ell},
  \]
  for some \(d \left( N_T + d \right)\) real values \(\left\{ \alpha^{(j \ell)}_u, (j, \ell) \in \llb 1, d \rrb^2, u \in \llb 0, N_T^{(j)} \rrb \right\}\),
  where for all \((j, \ell) \in \llb 1, d \rrb^2\):
    \[
      \begin{cases}
        r_\ell = \sum_{v = 1}^{N^{(\ell)}_T} \int_0^T k\left(\cdot, t - T_v^{(\ell)} \right) \indic{0 < t - T_v^{(\ell)} \le A} \dd t \\
        q_{u j \ell} = \sum_{v = 1}^{N^{(\ell)}_T} k\left(\cdot, T_u^{(j)} - T_v^{(\ell)} \right) \indic{0 < T_u^{(j)} - T_v^{(\ell)} \le A},
        \quad \forall u \in \llb 1, N_T^{(j)} \rrb.
      \end{cases}
    \]
\end{proposition}

The approximation \(L_0\) has been proposed by \citet{lemonnier_nonparametric_2014}.
It consists in
computing, for all \(j \in \llb 1, d \rrb\), \(\int_0^T \lambda^{(j)}(t) \dd t\) as if \(\varphi\) were the identity function (thus adding a negative contribution when \(\lambda^{(j)}\) is null).
In practice, this approximation is detrimental when there is a lot of inhibition (that is when \(\lambda^{(j)}\) is often null) but painless for exciting and moderately inhibiting processes.
In particular, if the process to estimate is a linear (exciting) Hawkes process, then it is very likely that \(L\) and \(L_0\) coincide around the true parameter, such that the estimator obtained with \(L_0\) is exactly that obtained by minimizing the negative log-likelihood with non-negativity constraints on \(g_{j \ell}\), for all \((j, \ell) \in \llb 1, d \rrb^2\).

Let us remark that the representer theorem does not hold directly for \(L\) (neither for least-squares inference based on \(J\)) because, for all \(j \in \llb 1, d \rrb\),
the term \(\int_0^T \lambda^{(j)}(t) \dd t\) (respectively \(\int_0^T \lambda^{(j)}(t)^2 \dd t\)) is no longer linear in \(h_{j \ell}\), \(\ell \in \llb 1, d \rrb\).
This shortcoming has already been observed for estimation of non-homogeneous processes using \rkhss \citep{flaxman_poisson_2017,yang_online_2017}.

\begin{proposition}[Representer theorem for approximate least-squares estimation]
  \label{prop:representer_thm_ls}
  The least-squares contrast \(J\) can be upper bounded by a function \(J^+ : \theta \in \Theta^+ \to \R\)
  such that if the approximated regularized least-squares problem
  \[
    \argmin_{\theta \in \Theta^+} J^+(\theta) + \frac{\eta}{2} \sum_{1 \le j, \ell \le d} \|h_{j \ell}\|_\H^2
  \]
  admits a minimizer,
  it has a solution \(\theta\) of the form given by \cref{prop:representer_thm_mle}.
\end{proposition}

Despite being based on approximations of the objective functions of interest,
both \cref{prop:representer_thm_mle,prop:representer_thm_ls} make explicit a representation of the functional parts of kernelized estimators of Hawkes processes.
Fortified by this positive result, we consider now semi-parametric candidates of the form given by \cref{prop:representer_thm_mle}, both for maximum likelihood and least-squares estimation,
and we thus define \(\Theta_\parallel^+\) to be the set of parameters \(\theta = \left( (\mu_j)_{1 \le j \le d}, (h_{j \ell})_{1 \le j, \ell \le d}, (b_{j \ell})_{1 \le j, \ell \le d} \right) \in \Theta^+\) such that for all \(1 \le j, \ell \le d\), \(h_{j \ell}\) has the form given by \cref{prop:representer_thm_mle}.
\cref{sec:implementation} presents the practical implementation of the proposed estimator.

	  \subsection{Kernelized expression of contrasts}
	    \label{sec:implementation}
According to the previous sections, we propose here to concede three approximations for estimating Hawkes processes with kernels:
i) interactions functions are assumed to have the form proposed by \cref{prop:representer_thm_mle}, that is \(h_{j \ell} = \alpha^{(j \ell)}_0 r_\ell + \sum_{u=1}^{N_T^{(j)}} \alpha^{(j \ell)}_u q_{u j \ell}\), for all \((j, \ell) \in \llb 1, d \rrb^2\);
ii) the integral term is replaced by a Riemann approximation, that is \(\int_0^T \lambda^{(j)}(t) \dd t \approx \frac{T}{M} \sum_{n=1}^M \lambda^{(j)}(\tau_n)\) for all \(j \in \llb 1, d \rrb\) (respectively for \(\int_0^T \lambda^{(j)}(t)^2 \dd t\))
with \(\tau_n = \frac{n-1}{M} T\);
iii) the ReLU function \(\varphi\) is replaced by the softplus function \(\tilde \varphi : x \mapsto \log ( 1 + \e^{\omega x}) / \omega\), where \(\omega > 0\) is a hyperparameter tuning the proximity of the approximation, which is a smooth upperbound of \(\varphi\).
This alteration, which ensures the differentiability of the objective function of \cref{opb:contrast}, is common in the literature,
regarding neural point processes \citep{mei_neural_2017,zuo_transformer_2020,zhang_selfattentive_2020,mei_transformer_2022},
but also for classification with the nondifferentiable hinge loss \citep{chapelle_training_2007,wang_distributed_2019,luo_learning_2021}.
Then, considering either \(\varphi_1 = \tilde \varphi\) and \(\varphi_2 = \log \circ \tilde \varphi\) (for the negative log-likelihood \(L\)),
or \(\varphi_1 = \tilde \varphi^2\) and \(\varphi_2 = 2 \tilde \varphi\) (for the least-squares criterion \(J\)),
both as point-wise functions,
the objective function of the regularized estimation \cref{opb:contrast} can be approximated by:
\begin{align*}
  \forall \theta \in \Theta_\parallel^+,\\
  F_{M, \omega}(\theta)
  &= \sum_{j=1}^d \left[ \frac{T}{M} \sum_{n=1}^M \varphi_1 \left( \mu_j + \sum_{\ell=1}^d \sum_{i = 1}^{N^{(\ell)}_T} \left( h_{j \ell} \left(\tau_n - T_i^{(\ell)} \right) + b_{j \ell} \right) \indic{0 < \tau_n - T_i^{(\ell)} \le A} \right) \right. \\
  &- \left. \sum_{n=1}^{N^{(j)}_T} \varphi_2 \left( \mu_j + \sum_{\ell=1}^d \sum_{i = 1}^{N^{(\ell)}_T} \left( h_{j \ell} \left(T_n^{(j)} - T_i^{(\ell)} \right) + b_{j \ell} \right) \indic{0 < T_n^{(j)} - T_i^{(\ell)} \le A} \right) \right]
  + \frac{\eta}{2} \sum_{1 \le j, \ell \le d} \|h_{j \ell}\|_\H^2 \\
  &= \sum_{j=1}^d \left[ \frac{T}{M} \ind ^\top \varphi_1 \left( \mu_j \ind + Q^{(j)} \alpha^{(j)} + B b^{(j)} \right)
  - \ind ^\top \varphi_2 \left( \mu_j \ind + K^{(j)} \alpha^{(j)} + E^{(j)} b^{(j)} \right) \right] \\
  &+ \frac{\eta}{2} \sum_{1 \le j, \ell \le d} { \alpha^{(j \ell)} } ^\top K^{(j \ell)} \alpha^{(j \ell)},
\end{align*}
where for all \(j \in \llb 1, d \rrb\),
\[
  \alpha^{(j)} =
  \begin{bmatrix}
    \alpha^{(j 1)}\\ \vdots \\ \alpha^{(j d)}
  \end{bmatrix} \in \R^{d (N_T^{(j)}+1)}
  \quad \text{and} \quad
  b^{(j)} =
  \begin{bmatrix}
    b_{j 1}\\ \vdots \\ b_{j d}
  \end{bmatrix} \in \R^d,
\]
and matrices are made explicit in \cref{app:implementation}.
Since \(\tilde \varphi\) is differentiable, this also holds true for \(\varphi_1\), \(\varphi_2\) and \(F_{M, \omega}\).
Then, gradients can be computed easily and have the form expressed in \cref{app:implementation}.

If the first of our three approximations is legitimate in nonparametric inference (as justified by the previous section), one may wonder how much the last two deteriorate the estimation. The forthcoming paragraph shows that the numerical impact is bounded by \(O(1/M) + O(1/\omega)\).
We think this is a reasonable price to pay to overcome the two obstacles of integration and nondifferential optimization.
For this statement to be quantified (by \cref{prop:approximation_mle,prop:approximation_ls} respectively for maximum likelihood estimation and least-squares minimization),
we move to Ivanov regularized optimization problems,
as it is common for statistical analysis of empirical risk minimizers (see for instance \citep{bartlett_rademacher_2002}).

For this purpose, let \(B > 0\) and \(C > 0\) be two fixed bounds,
and 
\(\Omega\) be the set of parameters \(\theta = \left( (\mu_j)_{1 \le j \le d}, (h_{j \ell})_{1 \le j, \ell \le d}, (b_{j \ell})_{1 \le j, \ell \le d} \right) \in \Theta^+\)
with bounded norms:
\[
  \Omega = \left\{ \theta \in \Theta^+ : \|\mu\|_\infty \le B, \|b\|_\infty \le B, \sqrt{\sum_{1 \le j, \ell \le d} \|h_{j \ell}\|_\H^2} \le C \right\}.
\]
We consider now \(L_{M, \omega}\) and \(J_{M, \omega}\)
to be the unregularized approximate objectives (intuitively \(F_{M, \omega}\) with \(\eta = 0\)):
\begin{align*}
  L_{M, \omega}(\theta) \text{ [respectively } J_{M, \omega}(\theta) \text{]}
  &= \sum_{j=1}^d \left[ \frac{T}{M} \sum_{n=1}^M \varphi_1 \left( \mu_j + \sum_{\ell=1}^d \sum_{i = 1}^{N^{(\ell)}_T} \left( h_{j \ell} \left(\tau_n - T_i^{(\ell)} \right) + b_{j \ell} \right) \indic{0 < \tau_n - T_i^{(\ell)} \le A} \right) \right. \\
  &- \left. \sum_{n=1}^{N^{(j)}_T} \varphi_2 \left( \mu_j + \sum_{\ell=1}^d \sum_{i = 1}^{N^{(\ell)}_T} \left( h_{j \ell} \left(T_n^{(j)} - T_i^{(\ell)} \right) + b_{j \ell} \right) \indic{0 < T_n^{(j)} - T_i^{(\ell)} \le A} \right) \right].
\end{align*}

\begin{proposition}[Approximation quality for maximum likelihood estimation]
  \label{prop:approximation_mle}
  Consider \cref{opb:contrast} as a maximum likelihood problem
  and let
  \[
    \hat \theta \in \argmin_{\theta \in \Omega} L_{M, \omega}(\theta)
    \quad \text{and} \quad
    \bar \theta \in \argmin_{\theta \in \Omega} L(\theta).
  \]
  Assume that the kernel \(k\) is bounded (by some \(\kappa > 0\)) and \(L_k\)-Lipschitz continuous (\(L_k > 0\)):
  \[
    \forall x \in \R:
    \quad
    k(x, x) \le \kappa^2
    \quad \text{and} \quad
    \forall (y, y') \in \R^2,
    |k(x, y) - k(x, y')| \le L_k |y - y'|,
  \]
  and that \(L(\hat \theta) < \infty\).
  Then, there exists \(\delta > 0\)
  such that:
  \begin{align*}
    0 \le
    L(\hat \theta) - L(\bar \theta)
    \le \frac{T}{M} \left( L_k C d T N_T + 4 (\kappa C + B) d N_T \right)
    + \frac{2 \log 2}{\omega} \left( d T + \frac{N_T}{\delta} \right).
  \end{align*}
\end{proposition}

\begin{proposition}[Approximation quality for least-squares estimation]
  \label{prop:approximation_ls}
  Consider \cref{opb:contrast} as a least-squares problem
  and let
  \[
    \hat \theta \in \argmin_{\theta \in \Omega} J_{M, \omega}(\theta)
    \quad \text{and} \quad
    \bar \theta \in \argmin_{\theta \in \Omega} J(\theta).
  \]
  Assume that the kernel \(k\) is bounded (by some \(\kappa > 0\)) and \(L_k\)-Lipschitz continuous (\(L_k > 0\)),
  and let \(H = 2 ( B + (\kappa C + B) N_T + \log 2)\).
  Then, for all \(\omega \ge 1\):
  \begin{align*}
    0 \le
    J(\hat \theta) - J(\bar \theta)
    \le \frac{HT}{M} \left( L_k C d T N_T + 4 (\kappa C + B) d N_T \right)
    + \frac{4 \log 2}{\omega} \left( H d T + N_T \right).
  \end{align*}
\end{proposition}

Both propositions tell that the difference between the proposed estimator \(\hat \theta\) (based on approximations) and the one coming from the original criterion \(\bar \theta\)
is controlled by two error terms.
The first one depends on the Riemann approximation and vanishes as soon as the number of bins \(M\) grows.
The second one is related to the link upper bound \(\tilde \varphi\),
and disappears when this approximation function comes closer to the ReLU function \(\varphi\) (\ie \(\omega\) becomes larger).

\cref{prop:approximation_mle} is based on two mild assumptions.
The first one, which also appears in \cref{prop:approximation_ls}, is using a bounded and Lipschitz continuous kernel \(k\).
This is the case for instance with the Gaussian kernel, \(k : (x, x') \in \R^2 \mapsto \e^{-\gamma (x-x')^2}\) (where \(\gamma > 0\)), with \(\kappa = 1\) and \(L_k = \sqrt \gamma\).
The second assumption says that the proposed estimator should have a finite log-likelihood, \ie that the estimated conditional intensities \(\lambda^{(j)}\) (\(j \in \llb 1, d \rrb\)) are not null at their times \(T_n^{(j)}\) (\(n \in \llb 1, N_T^{(j)} \rrb\)).
This is a rational requirement since a process cannot jump if its intensity is zero.


  \section{Numerical study}
    \label{sec:numerics}
\ms{Add online kernel hawkes, prediction application (see Transformer Hawkes process)}

\subsection{Synthetic data}

  This section aims first at assessing the impact of the approximation parameters \(\omega\) and \(M\),
  then at comparing our approach to the most related ones from the literature.
  For this purpose, we consider synthetic data coming from a \(3\)-variate Hawkes process with baseline intensities \(\mu_1 = \mu_2 = \mu_3 = 0.05\)
  and triggering kernels depicted in blue in \cref{img:toy_estimation} and defined below for all \(t \in \R_+\).
  Auto-interactions (\(g_{1 1}\), \(g_{2 2}\) and \(g_{3 3}\)) reflect the refractory phenomenon:
  \begin{align*}
    g_{1 1} (t) &= (8 t^2 - 1) \indic{t \le 0.5} + \e^{-2.5 (t-0.5)} \indic{t > 0.5} \\
    g_{2 2} (t) &= g_{3 3} (t) = (8 t^2 - 1) \indic{t \le 0.5} + \e^{-(t-0.5)} \indic{t > 0.5},
  \end{align*}
  while inter-interactions are either exciting or inhibiting:
  \begin{align*}
    g_{1 2} (t) &= \e^{-10 (t-1)^2} \\
    g_{1 3} (t) &= -0.6 \e^{-3 t^2} - 0.4 \e^{-3 (t-1)^2} \\
    g_{2 1} (t) &= 2^{-5t} \\
    g_{2 3} (t) &= -\e^{-2 (t-3)^2} \\
    g_{3 1} (t) &= -\e^{-5 (t-2)^2} \\
    g_{3 2} (t) &= (1 + \cos(\pi t)) \e^{-t}/2.
  \end{align*}
  
  \begin{figure}[ht]
    \center
    \vspace*{-1em}
    \includegraphics[width=.5\textwidth]{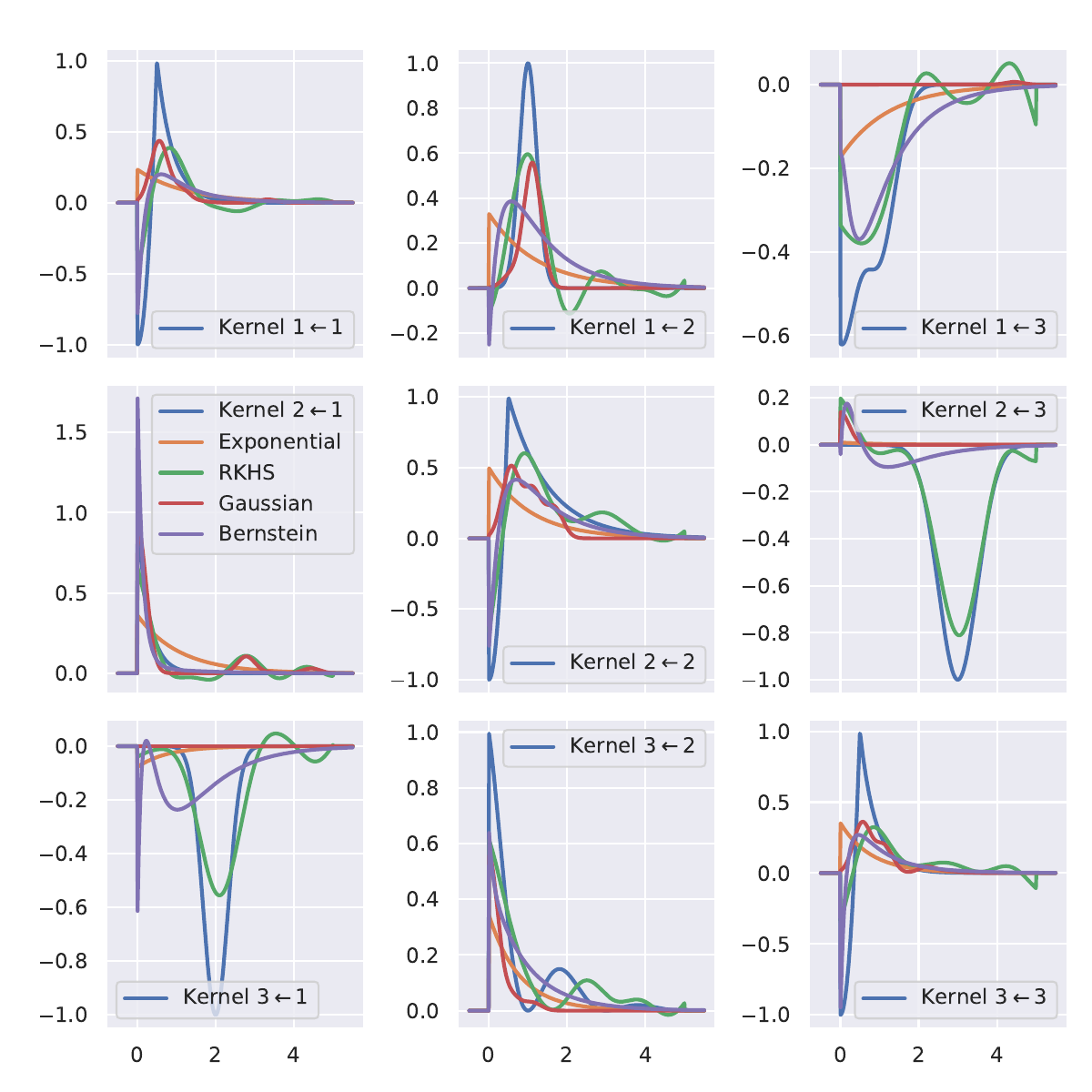}
    \vspace*{-2em}
    \caption{Example of estimations with horizon $T = 2000$.}
    \label{img:toy_estimation}
  \end{figure}
  
  The process is simulated thanks to the thinning method \citep{ogata_lewis_1981} with a burn-in period ensuring stationarity and a sufficiently large time window,
  then estimations are performed successively for a growing horizon \(T \in [250, 500, 1000, 2000]\) of the same observed trajectory (see \cref{img:toy_err,img:toy_loglik}).
  
  Methods compared are (see estimation examples in \cref{img:toy_estimation}):
  \begin{description}
    \item[Exponential:] the mainly used parametric model, for which triggering kernels are supposed to be exponential functions \citep{bonnet_inference_2023}.
    It allows for estimating exciting or inhibiting interactions;
    \item[RKHS:] our proposition (with a regularization coefficient \(\eta\)) based on the Gaussian kernel (parametrized by \(\gamma\)) and a support bound \(A = 5\);
    \item[Bernstein:] nonparametric estimation, for which triggering kernels are represented as a sum of exponential functions \citep{lemonnier_nonparametric_2014}:
    \(g_{j \ell}(t) = \sum_{u=1}^U \alpha_u^{(j \ell)} \e^{-\gamma u t}\) (with \(U = 10\)) and a quadratic penalty on the coefficients \(\alpha_u^{(j \ell)}\) controlled by the parameter \(\eta\) ;
    \item[Gaussian:] nonparametric estimation, for which triggering kernels are represented as a sum of Gaussian functions \citep{xu_learning_2016}:
    \(g_{j \ell}(t) = \sum_{u=1}^U \alpha_u^{(j \ell)} \e^{-\gamma (t - t_u)^2}\) (with \(U = 10\) and a regular grid on \([0, A]\)) and a quadratic penalty on the coefficients \(\alpha_u^{(j \ell)}\) controlled by the parameter \(\eta\).
  \end{description}
  Since the closest methods to ours are based on minimizing the negative log-likelihood, we only include our approach with the criterion \(L\).
  Numerical optimization is performed thanks to the L-BFGS-B method implemented in SciPy \citep{virtanen_scipy_2020}
  on a personal computer.
  In addition, let us remark that the method by \citet{lemonnier_nonparametric_2014} minimizes an approximation of the negative log-likelihood
  and that the one by \citet{xu_learning_2016} is restricted to exciting interactions (\(\alpha_u^{(j \ell)} \ge 0\)).
  
  All three nonparametric methods are tuned by parameters \(\gamma\) and \(\eta\), respectively chosen on the grids \([1, 10, 100]\) and \([0.1, 1, 10, 100]\) to maximize the log-likelihood computed on an independent validation trajectory (with same distribution as the training one).
  The numerical results presented below are computed on an independent test trajectory,
  and this whole procedure is repeated 10 times to produce the statistics (mean and 95\% confidence interval) reported in \cref{img:toy_app,img:toy_err,img:toy_loglik}.

  \begin{figure}[ht]
    \center
    \includegraphics[width=.5\textwidth]{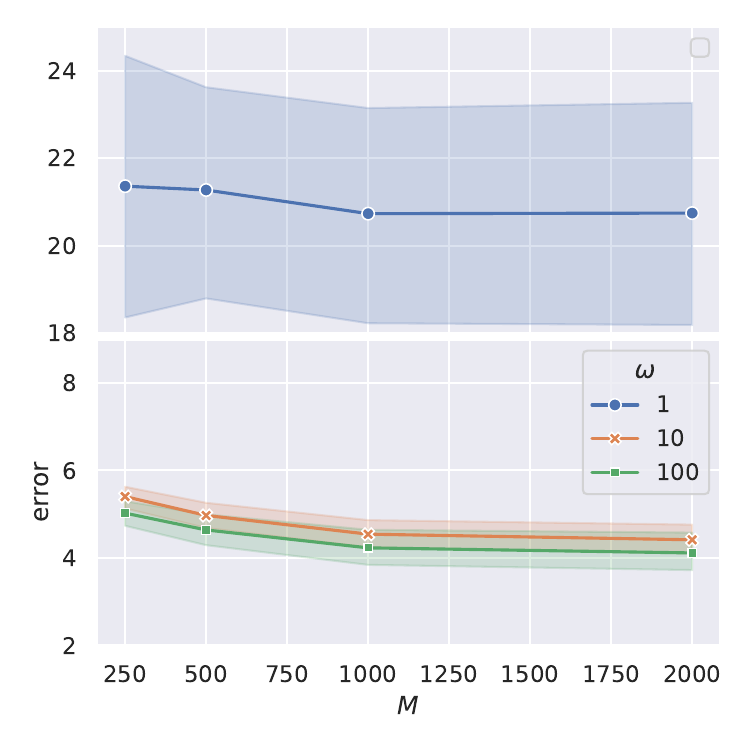}
    \caption{Approximation error of true kernels with respect to the hyperparameters $\omega$ and $M$.}
    \label{img:toy_app}
  \end{figure}
  
  As a first numerical experiment, \cref{img:toy_app} depicts the
  the smallest (over the grids of parameters \(\eta\) and \(\gamma\))
  sum of \(L^1\)-errors between the true triggering kernels and the estimation provided by our method \textbf{RKHS},
  for a horizon \(T = 1000\) and
  for different values of approximation parameters \(\omega\) and \(M\).
  As suggested by \cref{prop:approximation_mle}, the bigger \(\omega\) and the bigger \(M\), the lower the estimation error.
  This shows that the considered approximations are not statistically harmful when parameters are well chosen.
  For the subsequent numerical analysis, the values chosen are \(\omega = 100\) and \(M = \max(1000, 2 \max(N_T^{(1)}, N_T^{(2)}, N_T^{(3)}))\).

  \begin{figure}[ht]
    \center
    \includegraphics[width=.5\textwidth]{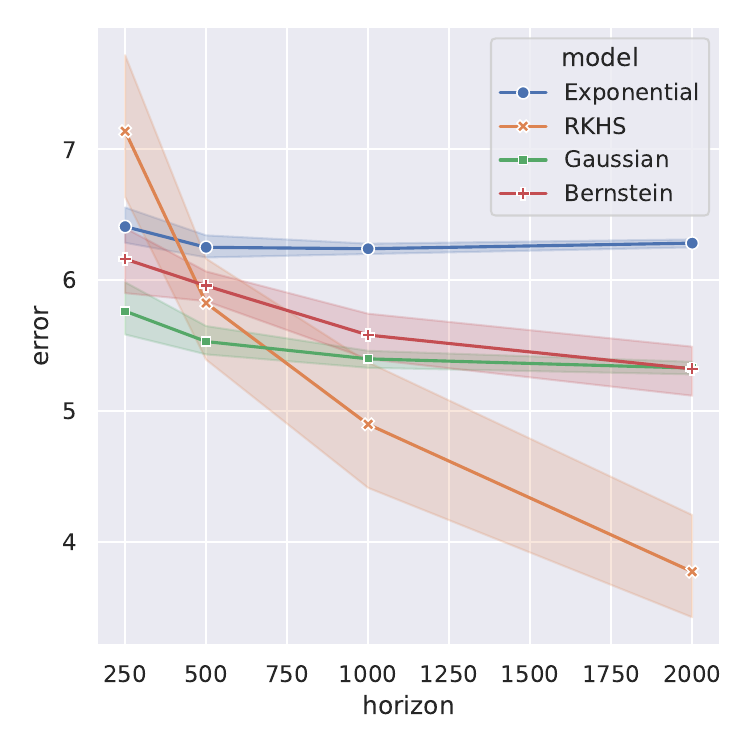}
    \caption{Approximation error of true kernels with respect to the horizon $T$.}
    \label{img:toy_err}
  \end{figure}

  \begin{figure}[ht]
    \center
    \includegraphics[width=.5\textwidth]{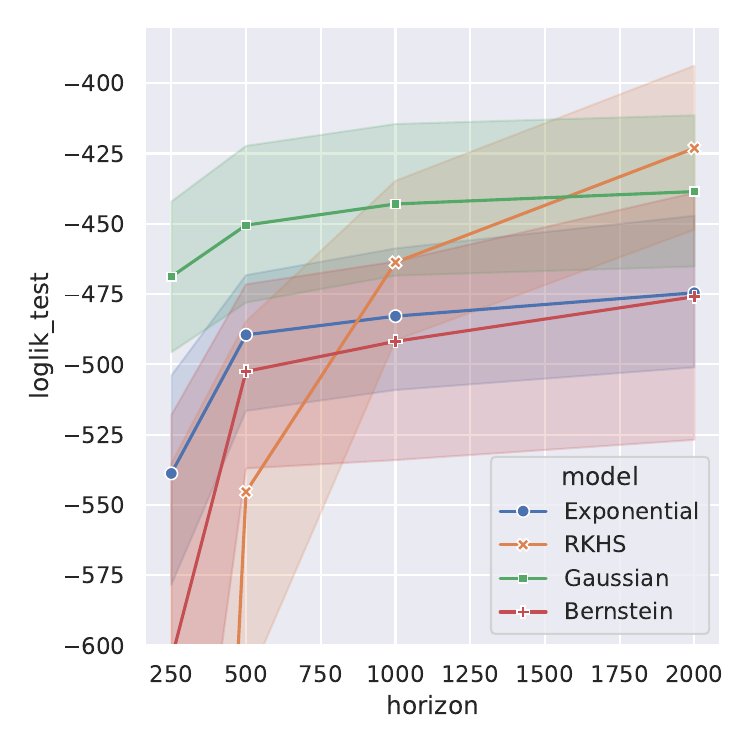}
    \caption{Test log-likelihood with respect to the horizon $T$.}
    \label{img:toy_loglik}
  \end{figure}
  
  Now, \cref{img:toy_err,img:toy_loglik} display, for the considered methods,
  respectively the sum of \(L^1\)-errors and the test log-likelihood,
  for the best model according to the validation log-likelihood.
  It appears clearly that competing methods perform poorly in estimating the chosen triggering kernels.
  Indeed, while able to infer inhibiting effects, \textbf{Exponential} is based on a misspecified model in view of the interaction functions to estimate, exhibiting non-exponential shapes.
  The same reasoning is true for \textbf{Gaussian} since it is designed only for exciting interactions.
  At last, \textbf{Bernstein} suffers from the approximated log-likelihood it optimizes, as well as the kind of functions it is able to approximate.
  In particular, \cref{img:toy_estimation} illustrates that it is able to recover auto-interactions with a refractory period but not cross-interactions.
  Let us remark that for \textbf{Gaussian} and \textbf{Bernstein}, augmenting the size \(U\) of the linear combination does not help because methods are limited by intrinsic misspecifications. 
  Overall, an \rkhs based method able to handle complex kernel functions, in particular that combine exciting and inhibiting interactions, with immediate or delayed effects, seems to be the best option.

\subsection{Neuronal data}
We illustrate our procedure on a neuronal dataset described in \citep{petersen_lognormal_2016,radosevic_decoupling_2019}, then analyzed in \citep{bonnet_inference_2023} with
\textbf{Exponential}
to study the neuronal activity of a red-eared turtle.
The full dataset contains 10 recordings of 40ms for $250$ neurons, that we preprocess as follows.
In order to work with a process that is almost stationary, we only consider the events that took place on the interval $[11, 24]$ms, using a similar setting as in \citep{bonnet_inference_2023}.
This procedure allows in particular to remove the effects of external stimuli that are performed before each recording.
For the sake of clarity when displaying interaction functions, we focus on a small subnetwork of five neurons, chosen randomly among the neurons with a large enough number of spikes (at least $500$ spikes in the concatenation of all recordings).
Parameters \(\gamma\) and \(\eta\) are chosen to maximize the likelihood on a validation trajectory obtained by a randomized concatenation of half of the recordings.

  \begin{figure*}[ht]
    \center
    \includegraphics[width=.8\textwidth]{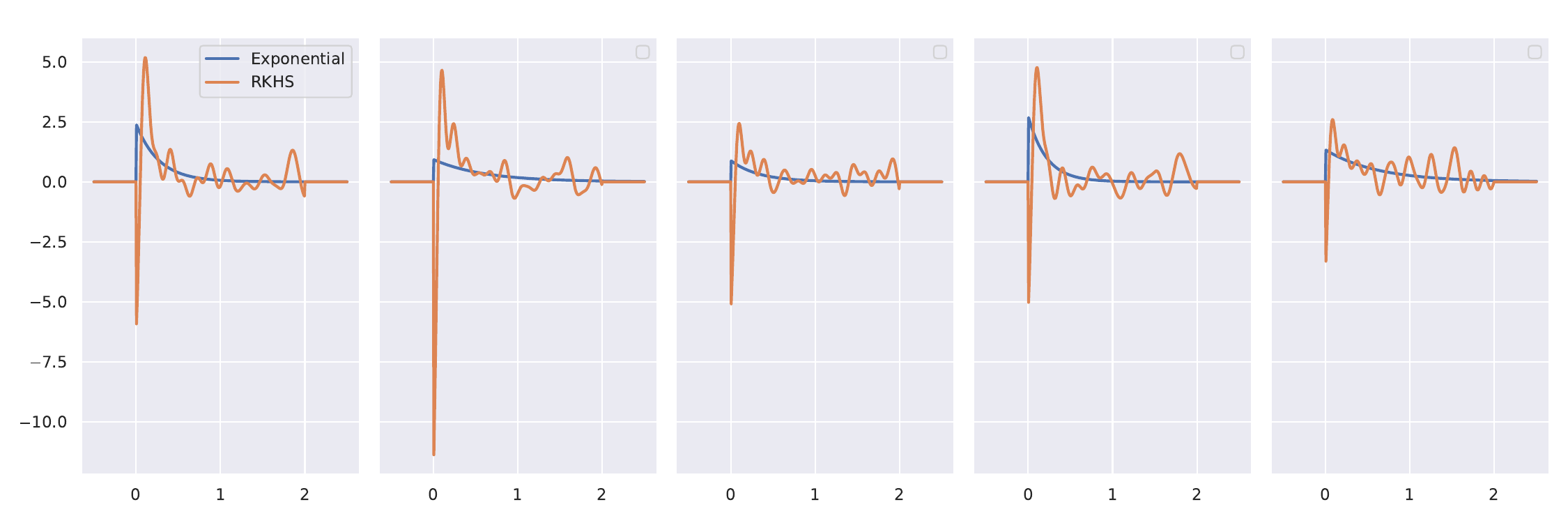}
    \caption{Auto-interaction functions learned on the neuronal subnetwork.}
    \label{img:neuron_diag}
  \end{figure*}

The five estimated auto-interaction functions are displayed in \cref{img:neuron_diag},
with \textbf{Exponential} as a reference.
The estimation provided by our method is consistent with current knowledge of neurons' behaviour, exhibiting first a large negative peak corresponding to the refractory period followed by a self-exciting effect.
As expected,
\textbf{Exponential}
cannot detect such a structure and only shows the self-exciting effect, which is probably underestimated due to the lack of refractory period modeling (see for instance the second graph where the exponential method only estimates a small self-interaction, which corresponds to an averaged version of
sharp negative and positive peaks).
We provide in \cref{img:neuron_all} of \cref{app:numerics}  
all 25 interaction functions
obtained with the four benchmarked methods.
We see that the two methods that are able to detect accurately the refractory period are \textbf{Bernstein} and \textbf{RKHS}, which is consistent with the results obtained on synthetic data. 
Interestingly, many cross-interactions appear to be low compared to auto-interactions.
Besides this qualitative assessment, \cref{tab:test_loglik} of \cref{app:numerics} gives the log-likelihood scores computed on a test trajectory obtained by a randomized concatenation of half of the recordings.
It shows that \textbf{RKHS} has the highest score, which leads to believe that it is more suited than competitors.
  

  \section{Discussion}
    \label{sec:discussion}
In this paper, we filled a gap in nonparametric inference of
Hawkes processes by introducing a batch learning method based on \rkhss, which is able to handle combinations of exciting and inhibiting effects.

On the one hand, as a kernel method, the proposed approach has for drawback its computational complexity, which, even if lower than that of neural networks based techniques, is higher than parametric or equivalent models.
This limitation could be lifted thanks to kernel approximation techniques (such as the Nyström method or random Fourier features \citep{yang_nystrom_2012})
or discretization schemes \citep{staerman_fadin_2023}.

On the other hand, we believe that the proposed method can serve as a stepping stone for many future developments,
including sparse learning of interaction functions,
the use of operator-valued kernels \citep{micchelli_learning_2005,paulsen_introduction_2016,brault_random_2016}
and extension to inference of spatiotemporal Hawkes processes \citep{li_multivariate_2024,siviero_flexible_2024}.


	\section*{Acknowledgement}
    The authors are grateful to the anonymous referees for their very relevant comments.

	\bibliography{maxime}
	
	\appendix
	
	\section{Proof of \cref{prop:duality}}
    \label{app:proof1}
  The optimization problem of interest can be written:
  \begin{opb*}
    \minimize{\theta \in \Theta,\, \atop (\xi_j)_{1 \le j \le d} \in (\R^M)^d}
    & \sum_{j=1}^d \left[ \frac{T}{M}  \| \xi_j \|^2 - 2 N_T^{(j)} \mu_j - 2 \sum_{\ell=1}^d \sum_{1 \le n \le N_T^{(j)} \atop 1 \le i \le N_T^{(\ell)}} \left( h_{j \ell} \left(T_n^{(j)} - T_i^{(\ell)} \right) + b_{j \ell} \right) \indic{0 < T_n^{(j)} - T_i^{(\ell)} \le A} \right] \\
    &+ \frac{\eta}{2} \sum_{1 \le j, \ell \le d} \|h_{j \ell}\|_\H^2 \\
    \st & \begin{starrayd}
      \forall (j, \ell) \in \llb 1, d \rrb^2\\
      \mu_j \ge 0
      &: \delta_j \ge 0\\
      \forall n \in \llb 1, P \rrb, ~ h_{j \ell}(x_n) + b_{j \ell} \ge 0
      &: {\beta_n^{(j \ell)}} \ge 0 \\
      \forall n \in \llb 1, M \rrb, ~ \mu_j + \sum_{\ell=1}^d \sum_{i=1}^{N_T^{(\ell)}} \left( h_{j \ell} \left(\tau_n - T_i^{(\ell)} \right) + b_{j \ell} \right) \indic{0 < \tau_n - T_i^{(\ell)} \le A} = \xi_{j n}
      &: {\alpha_n^{(j)}} \in \R,
    \end{starrayd}
  \end{opb*}
  where Lagrangian variables are indicated on the right.
  
  This problem is convex (in that it has a quadratic objective with affine constraints)
  and statisfies Slater's constraint qualification.
  As a consequence, the dual problem is attained for some dual variables \(\delta \in \R_+^d\), \({\beta^{(j \ell)}} \in \R_+^P\) and \({\alpha^{(j)}} \in \R^M\),
  and Karush-Kuhn-Tucker conditions hold for solutions \(\theta\) and \((\xi_j)_{1 \le j \le d}\),
  and \(\delta\), \({\beta^{(j \ell)}}\) and \({\alpha^{(j)}}\).
  In particular, by stationarity, for all \((j, \ell) \in \llb 1, d \rrb^2\):
  \begin{align*}
    h_{j \ell}
    &= \eta^{-1} \left[
      2 q_{j \ell}
      + \sum_{n=1}^P {\beta_n^{(j \ell)}} k \left( \cdot, x_n \right)
      - \sum_{n=1}^M {\alpha_n^{(j)}} r_{\ell n}
    \right],
  \end{align*}
  where
  \[
    \begin{cases}
      q_{j \ell} = \sum_{1 \le n \le N_T^{(j)} \atop 1 \le i \le N_T^{(\ell)}} k \left(\cdot, T_n^{(j)} - T_i^{(\ell)} \right) \indic{0 < T_n^{(j)} - T_i^{(\ell)} \le A} \\
      r_{\ell n} = \sum_{i=1}^{N_T^{(\ell)}} k \left( \cdot, \tau_n - T_i^{(\ell)} \right) \indic{0 < \tau_n - T_i^{(\ell)} \le A}.
    \end{cases}
  \]
  
  As a byproduct, the dual problem reads:
  \begin{opb*}
    \maximize{({\alpha^{(j)}})_{1 \le j \le d} \in (\R^M)^d,\, ({\beta^{(j \ell)}})_{1 \le j \le d} \in (\R^P)^{d \times d}}
    & -\sum_{j=1}^d \left[
      {\alpha^{(j)}}^\top \tilde K {\alpha^{(j)}} - \frac{2}{\eta} {\alpha^{(j)}}^\top \sum_{\ell=1}^d v_{j \ell}
    \right] \\
    &- \frac{1}{2 \eta} \sum_{1 \le j, \ell \le d} \left[
      {\beta^{(j \ell)}}^\top R {\beta^{(j \ell)}} + 2 {\beta^{(j \ell)}} ^\top \left( 2 u_{j \ell} - W_\ell {\alpha^{(j)}} \right)
    \right] \\
    \st & \begin{starray}
      \forall (j, \ell) \in \llb 1, d \rrb^2, {\beta^{(j \ell)}} \succcurlyeq 0 \\
      \forall j \in \llb 1, d \rrb, 2 N_T^{(j)} \le \ind^\top {\alpha^{(j)}} \\
      \forall (j, \ell) \in \llb 1, d \rrb^2, {\alpha^{(j)}}^\top z_\ell - \ind^\top {\beta^{(j \ell)}} = 2 \sum_{1 \le n \le N_T^{(j)} \atop 1 \le i \le N_T^{(\ell)}} \ind_{0 < T_n^{(j)} - T_i^{(\ell)} \le A},
    \end{starray}
  \end{opb*}
  where
  \(R = \left( k \left( x_n, x_{n'} \right) \right)_{1 \le n, n' \le P}\),
  \(\tilde K = \frac{M}{4T} I_M + \frac{1}{2\eta} \sum_{\ell=1}^d K_{\ell}\) (with \(I_M\) the identity matrix of size \(M\)),
  and for all \(\ell \in \llb 1, d \rrb\),
  \(K_\ell = \left( \iH{ r_{\ell n} }{ r_{\ell n'} } \right)_{1 \le n, n' \le M}\),
  \(W_\ell = \left( r_{\ell n'}(x_n) \right)_{1 \le n, n' \le M}\),
  \(z_\ell = \left( \sum_{i=1}^{N_T^{(\ell)}} \ind_{0 < \tau_n - T_i^{(\ell)} \le A} \right)_{1 \le n \le M}\),
  and for all \(j \in \llb 1, d \rrb\),
  \(u_{j \ell} = \left( q_{j \ell}(x_n) \right)_{1 \le n \le P}\), and
  \(v_{j \ell} = \left( \iH{ q_{j \ell} }{ r_{\ell n} } \right)_{1 \le n \le M}\).
  

	\section{Proof of \cref{prop:representer_thm_mle}}
    \label{app:proof2}
First of all, conditional intensities read:
\begin{align*}
  \forall j \in \llb 1, d \rrb,
  \forall t \in \R_+:
  \quad
  \lambda^{(j)}(t)
  &= \varphi \left( \mu_j + \sum_{\ell=1}^d \sum_{i = 1}^{N^{(\ell)}_t} g_{j \ell} \left(t - T_i^{(\ell)} \right) \right) \\
  &= \varphi \left( \mu_j + \sum_{\ell=1}^d \sum_{i = 1}^{N^{(\ell)}_T} \left( h_{j \ell} \left(t - T_i^{(\ell)} \right) + b_{j \ell} \right) \indic{0 < t - T_i^{(\ell)} \le A} \right).
\end{align*}
Then, the approximation \(L_0\) is \citep{lemonnier_nonparametric_2014}:
\[
  \forall \theta \in \Theta^+,
  \quad
  L_0(\theta) = \sum_{j=1}^d \left[ \int_0^T \left\{ \mu_j + \sum_{\ell=1}^d \sum_{i = 1}^{N^{(\ell)}_T} g_{j \ell} \left(t - T_i^{(\ell)} \right) \right\} \dd t - \sum_{n=1}^{N^{(j)}_T} \log \left( \lambda^{(j)} \left( T_n^{(j)} \right) \right) \right],
\]
which is a lower bound of \(L\) when \(\varphi\) is the ReLU function.
Now, let \(F : \Theta \to \R \cup \{\infty\}\) denote the objective function:
for all \(\theta \in \Theta^+\),
\begin{align*}
  F(\theta)
  &= L_0(\theta) + \frac{\eta}{2} \sum_{1 \le j, \ell \le d} \|h_{j \ell}\|_\H^2\\
  &= \sum_{j=1}^d \left[ \mu_j T + \sum_{\ell=1}^d \sum_{i = 1}^{N^{(\ell)}_T} \int_0^T h_{j \ell} \left(t - T_i^{(\ell)} \right) \indic{0 < t - T_i^{(\ell)} \le A} \dd t
  + \sum_{\ell=1}^d b_{j \ell} B^{(\ell)} \right. \\
  &- \left. \sum_{n=1}^{N^{(j)}_T} \log \left( \mu_j + \sum_{\ell=1}^d \sum_{i = 1}^{N^{(\ell)}_T} \left( h_{j \ell} \left(T_n^{(j)} - T_i^{(\ell)} \right) + b_{j \ell} \right) \indic{0 < T_n^{(j)} - T_i^{(\ell)} \le A} \right) \right]
  + \frac{\eta}{2} \sum_{1 \le j, \ell \le d} \|h_{j \ell}\|_\H^2,
\end{align*}
where \(B^{(\ell)} = \sum_{i = 1}^{N^{(\ell)}_T} \min \left(T-T_i^{(\ell)}, A \right)\).

  Let \(\theta\) be a minimizer of \(F\).
  By the reproducing property, we have:
  
  \begin{align*}
    F(\theta)
    &= \sum_{j=1}^d \left[ \mu_j T + \sum_{\ell=1}^d \sum_{i = 1}^{N^{(\ell)}_T} \int_0^T \iH{ h_{j \ell} }{ k\left(\cdot, t - T_i^{(\ell)} \right) \indic{0 < t - T_i^{(\ell)} \le A} } \dd t
    + \sum_{\ell=1}^d b_{j \ell} B^{(\ell)} \right. \\
    &- \left. \sum_{n=1}^{N^{(j)}_T} \log \left( \mu_j + \sum_{\ell=1}^d \sum_{i = 1}^{N^{(\ell)}_T} \left( \iH{ h_{j \ell} }{ k\left(\cdot, T_n^{(j)} - T_i^{(\ell)} \right) \indic{0 < T_n^{(j)} - T_i^{(\ell)} \le A} } + b_{j \ell} \indic{0 < T_n^{(j)} - T_i^{(\ell)} \le A} \right) \right) \right] \\
    &+ \frac{\eta}{2} \sum_{1 \le j, \ell \le d} \|h_{j \ell}\|_\H^2 \\
    &= \sum_{j=1}^d \left[ \mu_j T + \sum_{\ell=1}^d \int_0^T \iH{ h_{j \ell} }{ \sum_{i = 1}^{N^{(\ell)}_T} k\left(\cdot, t - T_i^{(\ell)} \right) \indic{0 < t - T_i^{(\ell)} \le A} } \dd t
    + \sum_{\ell=1}^d b_{j \ell} B^{(\ell)} \right. \\
    &- \left. \sum_{n=1}^{N^{(j)}_T} \log \left( \mu_j + \sum_{\ell=1}^d \left( \iH{ h_{j \ell} }{ \sum_{i = 1}^{N^{(\ell)}_T} k\left(\cdot, T_n^{(j)} - T_i^{(\ell)} \right) \indic{0 < T_n^{(j)} - T_i^{(\ell)} \le A} } + b_{j \ell} \sum_{i = 1}^{N^{(\ell)}_T} \indic{0 < T_n^{(j)} - T_i^{(\ell)} \le A} \right) \right) \right] \\
    &+ \frac{\eta}{2} \sum_{1 \le j, \ell \le d} \|h_{j \ell}\|_\H^2
  \end{align*}
  Now, for all \(\ell \in \llb 1, d \rrb\), \(L_\ell : h \in \H \mapsto \int_0^T \iH{ h }{ \sum_{i = 1}^{N^{(\ell)}_T} k\left(\cdot, t - T_i^{(\ell)} \right) \indic{0 < t - T_i^{(\ell)} \le A} } \dd t\)
  is a continuous linear operator from \(\H\) to \(\R\)
  so by the Riesz representation theorem, there exists \(r_\ell \in \H\) such that for all \(h \in \H\), \(L_\ell(h) = \iH{h}{r_\ell}\). 
  Moreover, by the reproducing property, for all \(x \in \R\):
  \begin{align*}
    r_\ell(x)
    = \iH{r_\ell}{k(\cdot, x)} 
    = L_\ell \left( k(\cdot, x) \right)
    &= \int_0^T \iH{ k(\cdot, x) }{ \sum_{i = 1}^{N^{(\ell)}_T} k\left(\cdot, t - T_i^{(\ell)} \right) \indic{0 < t - T_i^{(\ell)} \le A} } \dd t \\
    &= \int_0^T \sum_{i = 1}^{N^{(\ell)}_T} k\left(x, t - T_i^{(\ell)} \right) \indic{0 < t - T_i^{(\ell)} \le A} \dd t.     
  \end{align*}
  Then, denoting for all \(n \in \llb 1, N_T^{(j)} \rrb\):
  \[
    q_{n j \ell} = \sum_{i = 1}^{N^{(\ell)}_T} k\left(\cdot, T_n^{(j)} - T_i^{(\ell)} \right) \indic{0 < T_n^{(j)} - T_i^{(\ell)} \le A},
  \]
  the objective function reads:
  \begin{align*}
    F(\theta)
    &= \sum_{j=1}^d \left[ \mu_j T + \sum_{\ell=1}^d \iH{ h_{j \ell} }{ r_\ell }
    + \sum_{\ell=1}^d b_{j \ell} B^{(\ell)} \right. \\
    &- \left. \sum_{n=1}^{N^{(j)}_T} \log \left( \mu_j + \sum_{\ell=1}^d \left( \iH{ h_{j \ell} }{ q_{n j \ell} } + b_{j \ell} \sum_{i = 1}^{N^{(\ell)}_T} \indic{0 < T_n^{(j)} - T_i^{(\ell)} \le A} \right) \right) \right]
    + \frac{\eta}{2} \sum_{1 \le j, \ell \le d} \|h_{j \ell}\|_\H^2.
  \end{align*}
  Thus, still for all \((j, \ell) \in \llb 1, d \rrb^2\), 
  decomposing \(\H\) as the direct sum of the vector space \(\mathcal V_{j \ell}\) spanned by
  \(\left\{ r_\ell \right\} \cup \left\{ q_{u j \ell}, u \in \llb 1, N_T^{(j)} \rrb \right\}\)
  and its orthogonal subspace \(\mathcal V_{j \ell}^\perp\), we can write \(h_{j \ell} = h_{j \ell}^\parallel + h_{j \ell}^\perp\)
  with \(h_{j \ell}^\parallel \in \mathcal V_{j \ell}\) and \(h_{j \ell} \in \mathcal V_{j \ell}^\perp\).
  Then, for \(\theta^\parallel = \left( (\mu_j)_{1 \le j \le d}, (h_{j \ell}^\parallel)_{1 \le j, \ell \le d}, (b_{j \ell})_{1 \le j, \ell \le d} \right)\),
  we have \(\theta^\parallel \in \Theta\) and
  \begin{align*}
    \min_{\theta' \in \Theta} F(\theta')
    &= F(\theta)
    = L_0(\theta) + \frac{\eta}{2} \sum_{1 \le j, \ell \le d} \|h_{j \ell}\|_\H^2
    = L_0(\theta^\parallel) + \frac{\eta}{2} \sum_{1 \le j, \ell \le d} \|h_{j \ell}\|_\H^2 \\
    &\ge L_0(\theta^\parallel) + \frac{\eta}{2} \sum_{1 \le j, \ell \le d} \|h_{j \ell}^\parallel\|_\H^2
    = F(\theta^\parallel),
  \end{align*}
  by the Pythagorean theorem.
  This proves that \(\theta^\parallel\) is a minimizer of \(F\), which has the desired form:
  \[
    \forall (j, \ell) \in \llb 1, d \rrb^2:
    \quad
    h_{j \ell}^\parallel = \alpha^{(j \ell)}_0 r_\ell + \sum_{u=1}^{N_T^{(j)}} \alpha^{(j \ell)}_u q_{u j \ell},
  \]
  for some \((N_T^{(j)}+1)\) real values \(\alpha^{(j \ell)}_0, \dots, \alpha^{(j \ell)}_{N_T^{(j)}}\).

	\section{Proof of \cref{prop:representer_thm_ls}}
    \label{app:proof3}
Following \citet{lemonnier_nonparametric_2014}, we consider the approximation \(J_0\):
\[
  \forall \theta \in \Theta^+,
  \quad
  J_0(\theta) = \sum_{j=1}^d \left[ \int_0^T \left\{ \mu_j + \sum_{\ell=1}^d \sum_{i = 1}^{N^{(\ell)}_T} g_{j \ell} \left(t - T_i^{(\ell)} \right) \right\}^2 \dd t - 2 \sum_{n=1}^{N^{(j)}_T} \lambda^{(j)} \left( T_n^{(j)} \right) \right],
\]
which is an upper bound of \(J\) when \(\varphi\) is the ReLU function.
Now, it is enough to remark that for any candidate \(\theta \in \Theta^+\) and all \(j \in \llb 1, d \rrb\),
\begin{align} \label{equ:cs_ls}
  &\int_0^T \left\{ \mu_j + \sum_{\ell=1}^d \sum_{i = 1}^{N^{(\ell)}_T} g_{j \ell} \left(t - T_i^{(\ell)} \right) \right\}^2 \dd t \notag \\
  &= \int_0^T \left\{ \mu_j + \sum_{\ell=1}^d \sum_{i = 1}^{N^{(\ell)}_T} \left( h_{j \ell} \left(t - T_i^{(\ell)} \right) + b_{j \ell} \right) \indic{0 < t - T_i^{(\ell)} \le A} \right\}^2 \dd t \notag \\
  &= \mu_j^2 T
  + \int_0^T \left\{ \sum_{\ell=1}^d \left( \iH{h_{j \ell}}{\sum_{i=1}^{N_T^{(\ell)}} k(\cdot, t-T_i^{(\ell)}) \indic{0 < t - T_i^{(\ell)} \le A}} + \sum_{i = 1}^{N^{(\ell)}_T} b_{j \ell} \indic{0 < t - T_i^{(\ell)} \le A} \right) \right\}^2 \dd t \notag \\
  &+ 2 \mu_j \left( \sum_{\ell=1}^d \int_0^T \iH{ h_{j \ell} }{ \sum_{i = 1}^{N^{(\ell)}_T} k\left(\cdot, t - T_i^{(\ell)} \right) \indic{0 < t - T_i^{(\ell)} \le A} } \dd t
  + \sum_{\ell=1}^d b_{j \ell} B^{(\ell)} \right) \notag \\
  &\le \mu_j^2 T
  + \sum_{1 \le \ell, \ell' \le d} C_{\ell \ell'} \|h_{j \ell}\|_\H \|h_{j \ell'}\|_\H
  + \int_0^T \left( \sum_{\ell=1}^d b_{j \ell} \sum_{i = 1}^{N^{(\ell)}_T} \indic{0 < t - T_i^{(\ell)} \le A} \right)^2 \dd t \notag \\
  &+ 2 \sum_{1 \le \ell, \ell' \le d} D_{\ell \ell'} |b_{j \ell'}| \|h_{j \ell}\|_\H
  + 2 \mu_j \sum_{\ell=1}^d \int_0^T \iH{h_{j \ell}}{\sum_{i=1}^{N_T^{(\ell)}} k(\cdot, t-T_i^{(\ell)}) \indic{0 < t - T_i^{(\ell)} \le A}} \dd t \notag \\
  &+ 2 \mu_j \sum_{\ell=1}^d b_{j \ell} B^{(\ell)},
\end{align}
by Cauchy–Schwarz inequality,
where
\[
  C_{\ell \ell'} = \int_0^T \left \| \sum_{i=1}^{N_T^{(\ell)}} k(\cdot, t-T_i^{(\ell)}) \indic{0 < t - T_i^{(\ell)} \le A} \right\|_\H \left \| \sum_{i=1}^{N_T^{(\ell)}} k(\cdot, t-T_i^{(\ell')}) \indic{0 < t - T_i^{(\ell')} \le A} \right\|_\H \dd t
  \ge 0,
\]
and
\[
  D_{\ell \ell'}
  = \int_0^T \left \| \left( \sum_{i=1}^{N_T^{(\ell)}} k(\cdot, t-T_i^{(\ell)}) \indic{0 < t - T_i^{(\ell)} \le A} \right) \left( \sum_{i = 1}^{N^{(\ell')}_T} \indic{0 < t - T_i^{(\ell')} \le A} \right) \right\|_\H \dd t
  \ge 0.
\]
The upper bound \(J^+(\theta)\) consists in replacing in \(J_0(\theta)\) the left-hand side of \cref{equ:cs_ls} by its right-hand side.
The objective function then becomes:
for all \(\theta \in \Theta^+\),
\begin{align*}
  F(\theta)
  &= J^+(\theta) + \frac{\eta}{2} \sum_{1 \le j, \ell \le d} \|h_{j \ell}\|_\H^2\\
  &= \sum_{j=1}^d \left[ \mu_j^2 T + 2 \mu_j \sum_{\ell=1}^d b_{j \ell} B^{(\ell)}
  + \int_0^T \left( \sum_{\ell=1}^d b_{j \ell} \sum_{i = 1}^{N^{(\ell)}_T} \indic{0 < t - T_i^{(\ell)} \le A} \right)^2 \dd t \right. \\
  &+ 2 \mu_j \sum_{\ell=1}^d \int_0^T \iH{h_{j \ell}}{\sum_{i=1}^{N_T^{(\ell)}} k(\cdot, t-T_i^{(\ell)}) \indic{0 < t - T_i^{(\ell)} \le A}} \dd t
  + \sum_{1 \le \ell, \ell' \le d} \left( C_{\ell \ell'} \|h_{j \ell'}\|_\H + 2 D_{\ell \ell'} |b_{j \ell'}|\right) \|h_{j \ell}\|_\H \\
  &- \left. 2 \sum_{n=1}^{N^{(j)}_T} \varphi \left( \mu_j + \sum_{\ell=1}^d \left( \iH{ h_{j \ell} }{ \sum_{i = 1}^{N^{(\ell)}_T} k\left(\cdot, T_n^{(j)} - T_i^{(\ell)} \right) \indic{0 < T_n^{(j)} - T_i^{(\ell)} \le A} } + b_{j \ell} \sum_{i = 1}^{N^{(\ell)}_T} \indic{0 < T_n^{(j)} - T_i^{(\ell)} \le A} \right) \right) \right] \\
  &+ \frac{\eta}{2} \sum_{1 \le j, \ell \le d} \|h_{j \ell}\|_\H^2 \\
  &= \sum_{j=1}^d \left[ \mu_j^2 T + 2 \mu_j \sum_{\ell=1}^d b_{j \ell} B^{(\ell)}
  + \int_0^T \left( \sum_{\ell=1}^d b_{j \ell} \sum_{i = 1}^{N^{(\ell)}_T} \indic{0 < t - T_i^{(\ell)} \le A} \right)^2 \dd t \right. \\
  &+ 2 \mu_j \sum_{\ell=1}^d \iH{h_{j \ell}}{r_\ell}
  + \sum_{1 \le \ell, \ell' \le d} \left( C_{\ell \ell'} \|h_{j \ell'}\|_\H + 2 D_{\ell \ell'} |b_{j \ell'}|\right) \|h_{j \ell}\|_\H \\
  &- \left. 2 \sum_{n=1}^{N^{(j)}_T} \varphi \left( \mu_j + \sum_{\ell=1}^d \left( \iH{ h_{j \ell} }{ q_{n j \ell} } + b_{j \ell} \sum_{i = 1}^{N^{(\ell)}_T} \indic{0 < T_n^{(j)} - T_i^{(\ell)} \le A} \right) \right) \right]
  + \frac{\eta}{2} \sum_{1 \le j, \ell \le d} \|h_{j \ell}\|_\H^2.
\end{align*}

Now, the rest of the proof is similar to that of \cref{prop:representer_thm_mle}:
\(h_{j \ell}\) appears through \(\|h_{j \ell}\|_\H\) and the same linear terms as in the proof of \cref{prop:representer_thm_mle}, so the Pythagorean theorem still applies and the final form of \(h_{j \ell}\) is the same.

	\section{Implementation details}
	  \label{app:implementation}
From \cref{sec:implementation}, for all \(\theta \in \Theta_\parallel^+\), for each \((j, \ell) \in \llb 1, d \rrb^2\),
there exist parameters \(\left\{ \alpha_u^{(j \ell)} \right\}_{0 \le u \le N_T^{(j)}}\) such that:
\[
  h_{j \ell} = \sum_{u=0}^{N_T^{(j)}} \alpha_u^{(j \ell)} q_{u j \ell},
\]
where \(q_{0 j \ell} = r_\ell\).
Then
\begin{align*}
  F_M(\theta)
  &= \sum_{j=1}^d \left[ \frac{T}{M} \sum_{n=1}^M \varphi_1 \left( \mu_j + \sum_{\ell=1}^d \sum_{i = 1}^{N^{(\ell)}_T} \left( h_{j \ell} \left(\tau_n - T_i^{(\ell)} \right) + b_{j \ell} \right) \indic{0 < \tau_n - T_i^{(\ell)} \le A} \right) \right. \\
  &- \left. \sum_{n=1}^{N^{(j)}_T} \varphi_2 \left( \mu_j + \sum_{\ell=1}^d \sum_{i = 1}^{N^{(\ell)}_T} \left( h_{j \ell} \left(T_n^{(j)} - T_i^{(\ell)} \right) + b_{j \ell} \right) \indic{0 < T_n^{(j)} - T_i^{(\ell)} \le A} \right) \right]
  + \frac{\eta}{2} \sum_{1 \le j, \ell \le d} \|h_{j \ell}\|_\H^2 \\
  &= \sum_{j=1}^d \left[ \frac{T}{M} \sum_{n=1}^M \varphi_1 \left( \mu_j + \sum_{\ell=1}^d \sum_{u=0}^{N_T^{(j)}} \alpha_u^{(j \ell)} \iH{ q_{u j \ell} }{ \sum_{i = 1}^{N^{(\ell)}_T} k \left(\cdot, \tau_n - T_i^{(\ell)} \right) \indic{0 < \tau_n - T_i^{(\ell)} \le A} } \right. \right.\\
  &+ \left. \left. \sum_{\ell=1}^d b_{j \ell} \sum_{i = 1}^{N^{(\ell)}_T} \indic{0 < \tau_n - T_i^{(\ell)} \le A} \right) \right. \\
  &- \left. \sum_{n=1}^{N^{(j)}_T} \varphi_2 \left( \mu_j + \sum_{\ell=1}^d \sum_{u=0}^{N_T^{(j)}} \alpha_u^{(j \ell)} \iH{ q_{u j \ell}}{ q_{n j \ell} } + \sum_{\ell=1}^d b_{j \ell} \sum_{i = 1}^{N^{(\ell)}_T} \indic{0 < T_n^{(j)} - T_i^{(\ell)} \le A} \right) \right] \\
  &+ \frac{\eta}{2} \sum_{1 \le j, \ell \le d} \sum_{u=0}^{N_T^{(j)}} \sum_{n=0}^{N_T^{(j)}} \alpha_u^{(j \ell)} \alpha_n^{(j \ell)} \iH{ q_{u j \ell}}{ q_{n j \ell} } \\
  &= \sum_{j=1}^d \left[ \frac{T}{M} \ind ^\top \varphi_1 \left( \mu_j \ind + Q^{(j)} \alpha^{(j)} + B b^{(j)} \right)
  - \ind ^\top \varphi_2 \left( \mu_j \ind + K^{(j)} \alpha^{(j)} + E^{(j)} b^{(j)} \right) \right] \\
  &+ \frac{\eta}{2} \sum_{1 \le j, \ell \le d} { \alpha^{(j \ell)} } ^\top K^{(j \ell)} \alpha^{(j \ell)},
\end{align*}
where
\[
  B = \left( \sum_{i = 1}^{N^{(\ell)}_T} \indic{0 < \tau_n - T_i^{(\ell)} \le A} \right)_{1 \le n \le M \atop 1 \le \ell \le d} \in \R^{M \times d},
\]
and for all \(j \in \llb 1, d \rrb\),
\[
  E^{(j)} = \left( \sum_{i = 1}^{N^{(\ell)}_T} \indic{0 < T_n^{(j)} - T_i^{(\ell)} \le A} \right)_{1 \le n \le N_T^{(j)} \atop 1 \le \ell \le d} \in \R^{N_T^{(j)} \times d},
\]
and
\[
  \alpha^{(j)} =
  \begin{bmatrix}
    \alpha^{(j 1)}\\ \vdots \\ \alpha^{(j d)}
  \end{bmatrix} \in \R^{d (N_T^{(j)}+1)}
  \quad \text{and} \quad
  b^{(j)} =
  \begin{bmatrix}
    b_{j 1}\\ \vdots \\ b_{j d}
  \end{bmatrix} \in \R^d,
\]
are concatenated vectors,
\[
  Q^{(j)} = \left[ Q^{(j 1)} \mid \dots \mid Q^{(j d)} \right] \in \R^{M \times d (N_T^{(j)}+1)}
  \quad \text{and} \quad
  K^{(j)} = \left[ K^{(j 1)}_{1} \mid \dots \mid K^{(j d)}_{1} \right] \in \R^{N_T^{(j)} \times d(N_T^{(j)}+1)},
\]
are concatenated matrices
with, for all \(\ell \in \llb 1, d \rrb\),
\begin{align*}
  Q^{(j \ell)}
  = \left(
      \iH{ q_{u j \ell} }{ \sum_{i = 1}^{N^{(\ell)}_T} k \left(\cdot, \tau_n - T_i^{(\ell)} \right) \indic{0 < \tau_n - T_i^{(\ell)} \le A} }
    \right)_{1 \le n \le M \atop 0 \le u \le N_T^{(j)}},
  \quad
  K^{(j \ell)}
  = \left(
      \iH{ q_{u j \ell}}{ q_{n j \ell} }
    \right)_{0 \le n \le N_T^{(j)} \atop 0 \le u \le N_T^{(j)}},
\end{align*}
and \(K^{(j \ell)}_{1}\) is the submatrix composed of the last \(n\) rows of \(K^{(j \ell)}\).

Now, for all
\(u \in \llb 0, N_T^{(j)} \rrb\),
\[
  q_{u j \ell} = \left\{ \sum_{v = 1}^{N^{(\ell)}_T} \int_0^T k\left(\cdot, t - T_v^{(\ell)} \right) \indic{0 < t - T_v^{(\ell)} \le A} \dd t \right\} \indic{u=0}
  + \left\{ \sum_{v = 1}^{N^{(\ell)}_T} k\left(\cdot, T_u^{(j)} - T_v^{(\ell)} \right) \indic{0 < T_u^{(j)} - T_v^{(\ell)} \le A} \right\} \indic{u \ge 1},
\]
so for all \(n \in \llb 1, M \rrb\),
\begin{align*}
  Q^{(j \ell)}_{n u}
  &= \left\{ \int_0^T s_\ell \left(\tau_n, t \right) \dd t \right\} \indic{u=0}
  + s_\ell \left(\tau_n, T_u^{(j)} \right) \indic{u\ge1},
\end{align*}
and for all \(n \in \llb 0, N_T^{(j)} \rrb\),
\begin{align*}
  K^{(j \ell)}_{n u}
  &= \left\{ \int_0^T \int_0^T s_\ell \left(t, t' \right) \dd t \dd t' \right\} \indic{u=0} \indic{n=0}
  + \left\{ \int_0^T s_\ell \left(t, T_u^{(j)} \right) \dd t \right\} \indic{u\ge1} \indic{n=0} \\
  &+ \left\{ \int_0^T s_\ell \left(T_n^{(j)}, t \right) \dd t \right\} \indic{u=0} \indic{n\ge1} 
  + s_\ell \left(T_n^{(j)}, T_u^{(j)} \right) \indic{u\ge1} \indic{n\ge1}, \\
\end{align*}
where
\[
  s_\ell : (x, x') \in \R^2 \mapsto \sum_{1 \le i, v \le N^{(\ell)}_T} k \left(x - T_i^{(\ell)}, x' - T_v^{(\ell)} \right) \indic{0 < x - T_i^{(\ell)} \le A} \indic{0 < x' - T_v^{(\ell)} \le A}.
\]

\begin{remark}
  In the particular case of the Gaussian kernel, \(k : (x, x') \in \R^2 \mapsto \e^{-\gamma (x-x')^2}\) (where \(\gamma > 0\)),
  for all \(\ell \in \llb 1, d \rrb^2\),
  \begin{align*}
    \forall x \in \R:
    \quad
    \int_0^T s_\ell(x, t) \dd t
    = \frac{\sqrt \pi}{2} \sum_{1 \le i, v \le N^{(\ell)}_T}
    &\left[
      \erf_\gamma \left( \min \left( T-T_v^{(\ell)}, A \right) - \left( x - T_i^{(\ell)} \right) \right)
    \right. \\
    &\left.
      + \erf_\gamma \left( x - T_i^{(\ell)} \right)
    \right] \indic{0 < x - T_i^{(\ell)} \le A},\\
  \end{align*}
  where \(\erf_\gamma : x \mapsto \gamma^{-1/2} \erf (\gamma^{1/2} x)\) and \(\erf\) is the Gauss error function.
  Moreover,
  \begin{align*}
    \int_0^T \int_0^T s_\ell \left(t, t' \right) \dd t \dd t'
    = \frac{\sqrt \pi}{2} \sum_{1 \le i, v \le N^{(\ell)}_T}
    &\left[
      2 G_\gamma \left( \min \left( T-T_v^{(\ell)}, A \right) \right)
    \right. \\
    &\left.
      - G_\gamma \left( \min \left( T-T_v^{(\ell)}, A \right) - \min \left( T-T_i^{(\ell)}, A \right) \right)
    \right],
  \end{align*}
  where \(G_\gamma  : x \in \R \mapsto x \erf_\gamma (x) + \frac{\gamma^{-1}}{\sqrt \pi} \left( \e^{- \gamma x^2} - 1 \right)\) is an antiderivative of \(\erf_\gamma\).
\end{remark}

If \(\varphi\) is differentiable, this also holds true for \(\varphi_1\), \(\varphi_2\) and \(F_M\),
and for all \(j \in \llb 1, d \rrb\) and all \(\theta \in \Theta^+_\parallel\), gradients read:
\begin{align*}
  \frac{\partial F_M}{\partial \mu_j} (\theta)
  = \frac{T}{M} \ind ^\top \varphi_1' \left( \mu_j \ind + Q^{(j)} \alpha^{(j)} + Bb^{(j)} \right)
  - \ind ^\top \varphi_2' \left( \mu_j \ind + K^{(j)} \alpha^{(j)} + E^{(j)} b^{(j)} \right),
\end{align*}
and
\begin{align*}
  \nabla_{\alpha^{(j)}} F_M (\theta)
  = \frac{T}{M} {Q^{(j)}} ^\top \varphi_1' \left( \mu_j \ind + Q^{(j)} \alpha^{(j)} + Bb^{(j)} \right)
  - {K^{(j)}} ^\top \varphi_2' \left( \mu_j \ind + K^{(j)} \alpha^{(j)} + E^{(j)} b^{(j)} \right)
  + \eta
  \begin{bmatrix}
    K^{(j 1)} \alpha^{(j 1)} \\ \vdots \\ K^{(j d)} \alpha^{(j d)}
  \end{bmatrix},
\end{align*}
and
\begin{align*}
  \nabla_{b^{(j)}} F_M (\theta)
  = \frac{T}{M} B ^\top \varphi_1' \left( \mu_j \ind + Q^{(j)} \alpha^{(j)} + Bb^{(j)} \right)
  - {E^{(j)}} ^\top \varphi_2' \left( \mu_j \ind + K^{(j)} \alpha^{(j)} + E^{(j)} b^{(j)} \right).
\end{align*}


	\section{Proofs of \cref{prop:approximation_mle,prop:approximation_ls}}
	  \label{app:approximation}
Throughout this section, we assume that the kernel assotiated to the \rkhs \(\H\) is
bounded:
\[
  \exists \kappa > 0 :
  \quad
  \forall x \in \R,
  k(x, x) \le \kappa^2,
\]
and
\(L_k\)-Lipschitz continuous:
\[
  \forall x \in \R:
  \quad
  \forall (y, y') \in \R^2,
  |k(x, y) - k(x, y')| \le L_k |y - y'|.
\]
We also assume that \(\varphi\) is
the ReLU function.

\begin{lemma}\label{lem:approximation_bound}
  Let \(\delta > 0\).
  Then,
  \[
    \forall x \in \R,
    \quad
    \left| \varphi(x) - \tilde \varphi(x) \right| \le \frac{\log 2}{\omega},
  \]
  and
  \[
    \forall x > \delta,
    \quad
    \left| (\log \circ \varphi)(x) - (\log \circ \tilde \varphi)(x) \right| \le \frac{\log 2}{\delta \omega}.
  \]
\end{lemma}

\begin{proof}
  Let \(x \in \R\).
  If \(x < 0\), then
  \[
     \left| \varphi(x) - \tilde \varphi(x) \right|
     = \tilde \varphi(x)
     \le \tilde \varphi(0)
     = \frac{\log 2}{\omega}.
  \]
  If \(x \ge 0\), then
  \[
     \left| \varphi(x) - \tilde \varphi(x) \right|
     = \frac{\log(1 + \e^{-\omega x})}{\omega}
     \le \frac{\log 2}{\omega}.
  \]
  Let \(\delta > 0\) and \(x > \delta\).
  Then
  \[
    \left| (\log \circ \varphi)(x) - (\log \circ \tilde \varphi)(x) \right|
    = \log \left( \frac{\tilde \varphi (x)}{x}\right)
    = \log \left( 1 + \frac{\log(1 + \e^{-\omega x})}{\omega x} \right)
    \le \frac{\log(1 + \e^{-\omega x})}{\omega x}
    \le \frac{\log 2}{\delta \omega}.
  \]
\end{proof}

\begin{lemma}\label{lem:approximation_lip}
  Let \(\theta \in \Omega\).
  Then, for all \(j \in \llb 1, d \rrb\)
  and for every interval \(I\) without a discountinuity of \(\lambda^{(j)}\),
  \(\lambda^{(j)}\) is \(L_\lambda\)-Lipschitz continuous,
  where \(L_\lambda = L_k N_T C\).
\end{lemma}

\begin{proof}
  Let \(\theta \in \Theta^+\) and \(j \in \llb 1, d \rrb\),
  and \(I\) an interval without a jump time.
  For all \((t, t') \in I^2\),
  \begin{align*}
    |\lambda^{(j)}(t) - \lambda^{(j)}(t')|
    &= \left| \varphi \left( \mu_j + \sum_{\ell=1}^d \sum_{i = 1}^{N^{(\ell)}_t} g_{j \ell} \left(t - T_i^{(\ell)} \right) \right) - \varphi \left( \mu_j + \sum_{\ell=1}^d \sum_{i = 1}^{N^{(\ell)}_t} g_{j \ell} \left(t' - T_i^{(\ell)} \right) \right) \right| \\
    &\le \left| \sum_{\ell=1}^d \sum_{i = 1}^{N^{(\ell)}_t} g_{j \ell} \left(t - T_i^{(\ell)} \right) - \sum_{\ell=1}^d \sum_{i = 1}^{N^{(\ell)}_t} g_{j \ell} \left(t' - T_i^{(\ell)} \right)\right| \\
    &\le \sum_{\ell=1}^d \sum_{i = 1}^{N^{(\ell)}_t} \left| \left( h_{j \ell} \left(t - T_i^{(\ell)} \right) + b_{j \ell} \right) \indic{0 < t - T_i^{(\ell)} \le A} - \left( h_{j \ell} \left(t' - T_i^{(\ell)} \right) + b_{j \ell} \right) \indic{0 < t' - T_i^{(\ell)} \le A} \right| \\
    &= \sum_{\ell=1}^d \sum_{i = 1}^{N^{(\ell)}_t} \left| h_{j \ell} \left(t - T_i^{(\ell)} \right) - h_{j \ell} \left(t' - T_i^{(\ell)} \right) \right| \indic{0 < t - T_i^{(\ell)} \le A} \\
    &\le \sum_{\ell=1}^d \sum_{i = 1}^{N^{(\ell)}_t} \|h_{j \ell}\|_\H \left\|k \left(\cdot, t - T_i^{(\ell)} \right) - k \left(\cdot, t' - T_i^{(\ell)} \right) \right\|_\H \indic{0 < t - T_i^{(\ell)} \le A} \\
    &\le \sum_{\ell=1}^d \sum_{i = 1}^{N^{(\ell)}_t} L_k \|h_{j \ell}\|_\H |t - t'| \indic{0 < t - T_i^{(\ell)} \le A} \\
    &\le L_k \sum_{\ell=1}^d N^{(\ell)}_T \|h_{j \ell}\|_\H |t - t'|.
  \end{align*}
  Thus, \(\lambda^{(j)}\) is \(\left( L_k \sum_{\ell=1}^d N^{(\ell)}_T \|h_{j \ell}\|_\H \right)\)-Lipschitz continuous.
  Now, for \(\theta \in \Omega\),
  \begin{align*}
    L_k \sum_{\ell=1}^d N^{(\ell)}_T \|h_{j \ell}\|_\H
    \le L_k N_T C
    = L_\lambda.
  \end{align*}
\end{proof}

\begin{lemma}\label{lem:approximation_int}
  Let \(\theta \in \Omega\), \(j \in \llb 1, d \rrb\)
  and \(I\) an interval (of length \(|I|\)) with \(S_I\) discountinuities of \(\lambda^{(j)}\).
  Then, for any \(\tau \in I\),
  \[
    \int_I \left| \lambda^{(j)}(t) - \lambda^{(j)}(\tau) \right| \dd t
    \le \frac{L_\lambda |I|^2}{2} + |I| S_I G,
  \]
  where \(L_\lambda = L_k N_T C\)
  and \(G = \kappa C + B\).
\end{lemma}

\begin{proof}
  Let \(\theta \in \Omega\).
  Since discountinuities of \(\lambda^{(j)}\) come from discountinuities of the interactions functions \(g_{j \ell}\) (\(\ell \in \llb 1, d \rrb\)) at the boundaries of their support \([0, A]\) (the functions \(h_{j \ell}\), \(\ell \in \llb 1, d \rrb\), are continuous by Lipschitz-continuity of the kernel \(k\)),
  their amplitudes are bounded by
  \[
    \max_{1 \le \ell \le d} \max(|g_{j \ell}(0), |g_{j \ell}(A)|)
    \le \max_{1 \le \ell \le d} \|g_{j \ell}\|_\infty.
  \]
  But for all \(\ell \in \llb 1, d \rrb\) and \(x \in \R\),
  \begin{align*}
    |g_{j \ell}(x)|
    \le |\iH{h_{j \ell}}{k(\cdot, x)}| + |b_{j \ell}|
    \le \|h_{j \ell}\|_\H \sqrt{k(x, x)} + |b_{j \ell}|
    \le \kappa C + B
    = G.
  \end{align*}
  So discountinuities have jumps bounded by \(G\).
  
  Let \(\tau \in I\).
  Then, for all \(t \in I\), by \cref{lem:approximation_lip}:
  \begin{align*}
    \left| \lambda^{(j)}(t) - \lambda^{(j)}(\tau) \right| \le L_\lambda |t - \tau| + S_I G.
  \end{align*}
  By integration, it comes
  \[
    \int_I \left| \lambda^{(j)}(t) - \lambda^{(j)}(\tau) \right| \dd t
    \le L_\lambda \int_0^{|I|} t \dd t + |I| S_I G
    = \frac{L_\lambda|I|^2}{2} + |I| S_I G.
  \]
\end{proof}

\subsubsection*{Proof of \cref{prop:approximation_mle}}
  Denoting, for all \(j \in \llb 1, d \rrb\), \(\tilde \lambda^{(j)} (t) = \tilde \varphi \left( \mu_j + \sum_{\ell=1}^d \sum_{i = 1}^{N^{(\ell)}_T} \left( h_{j \ell} \left(t - T_i^{(\ell)} \right) + b_{j \ell} \right) \indic{0 < t - T_i^{(\ell)} \le A} \right)\), we have:
  \begin{align*}
    0
    &\le L(\hat \theta) - L(\bar \theta) \\
    &= L(\hat \theta) - L_{M, \omega}(\hat \theta) + L_{M, \omega}(\hat \theta) - L_{M, \omega}(\bar \theta) + L_{M, \omega}(\bar \theta) - L(\bar \theta) \\
    &\le 2 \max_{\theta \in \{\hat \theta, \bar \theta\}} \left| L(\theta) - L_{M, \omega}(\theta) \right| \\
    &\le 2 \max_{\theta \in \{\hat \theta, \bar \theta\}} \sum_{j=1}^d 
    \left|
      \left[ \int_0^T \lambda^{(j)}(t) \dd t - \frac{T}{M} \sum_{n=1}^M \tilde \lambda^{(j)}(\tau_n) \right]
      - \left[ \sum_{n=1}^{N^{(j)}_T} \log \left( \lambda^{(j)} \left( T_n^{(j)} \right) \right) - \sum_{n=1}^{N^{(j)}_T} \log \left( \tilde \lambda^{(j)} \left( T_n^{(j)} \right) \right) \right]
    \right|.
  \end{align*}
  Let \(\theta\) be either \(\hat \theta\) or \(\bar \theta\).
  Regarding the first term,
  \begin{align*}
    \left| \int_0^T \lambda^{(j)}(t) \dd t - \frac{T}{M} \sum_{n=1}^M \tilde \lambda^{(j)}(\tau_n) \right|
    &= \left| \sum_{n=1}^M \int_{\tau_n}^{\tau_n + \frac TM} \left( \lambda^{(j)}(t) - \tilde \lambda^{(j)}(\tau_n) \right) \dd t \right| \\
    &= \left| \sum_{n=1}^M \left(
    \int_{\tau_n}^{\tau_n + \frac TM} \left( \lambda^{(j)}(t) - \lambda^{(j)}(\tau_n) \right) \dd t
    + \frac TM \left( \lambda^{(j)}(\tau_n) - \tilde \lambda^{(j)}(\tau_n) \right) \right) \right| \\
    &\le \sum_{n=1}^M \left| \int_{\tau_n}^{\tau_n + \frac TM} \left( \lambda^{(j)}(t) - \lambda^{(j)}(\tau_n) \right) \dd t \right|
    + \frac{T \log 2}{\omega},
  \end{align*}
  by \cref{lem:approximation_bound}.
  Now, by \cref{lem:approximation_int}, denoting \(L_\lambda = L_k N_T C\) and \(G = \kappa C + B\):
  \begin{align*}
    \left| \int_0^T \lambda^{(j)}(t) \dd t - \frac{T}{M} \sum_{n=1}^M \tilde \lambda^{(j)}(\tau_n) \right|
    &\le \sum_{n=1}^M \left( \frac{L_\lambda T^2}{2 M^2} + \frac{T}{M} S_{[\tau_n, \tau_n + \frac TM)} G \right)
    + \frac{T \log 2}{\omega} \\
    &= \frac{L_\lambda T^2}{2 M} + \frac{T}{M} S_{[0, T)} G + \frac{T \log 2}{\omega} \\
    &\le \frac{L_\lambda T^2}{2 M} + \frac{2 N_T G T}{M} + \frac{T \log 2}{\omega},
  \end{align*}
  since \(\lambda^{(j)}\) has at most \(2 N_T\) discountinuities on \([0, T]\) (since the functions \(h_{j \ell}\), \(\ell \in \llb 1, d \rrb\), are continuous --by Lipschitz-continuity of the kernel \(k\)--, discountinuities of \(\lambda^{(j)}\) come from discountinuities of the interactions functions \(g_{j \ell}\), \(\ell \in \llb 1, d \rrb\),  at the boundaries of their support \([0, A]\)).
  
  In addition, regarding the second term, since \(L(\theta) < \infty\) (by assumption \(L(\hat \theta) < \infty\) and \(L(\bar \theta) < \infty\), since \(F\) is proper), there exists \(\delta > 0\) such that
  \[
    \min_{\theta \in \{\hat \theta, \bar \theta\},\, j \in \llb 1, d \rrb, \, n \in \llb 1, N_T^{(j)} \rrb} \left\{ \mu_j + \sum_{\ell=1}^d \sum_{i = 1}^{N^{(\ell)}_T} \left( h_{j \ell} \left( T_n^{(j)} - T_i^{(\ell)} \right) + b_{j \ell} \right) \indic{0 < t - T_i^{(\ell)} \le A} \right\} > \delta.
  \]
  It results that:
  \begin{align*}
   & \left| \sum_{n=1}^{N^{(j)}_T} \log \left( \lambda^{(j)} \left( T_n^{(j)} \right) \right) - \sum_{n=1}^{N^{(j)}_T} \log \left( \tilde \lambda^{(j)} \left( T_n^{(j)} \right) \right) \right| \\
    &\le \sum_{n=1}^{N^{(j)}_T} \left| (\log \circ \varphi) \left( \mu_j + \sum_{\ell=1}^d \sum_{i = 1}^{N^{(\ell)}_T} \left( h_{j \ell} \left( T_n^{(j)} - T_i^{(\ell)} \right) + b_{j \ell} \right) \indic{0 < T_n^{(j)} - T_i^{(\ell)} \le A} \right) \right. \\
    &- \left. (\log \circ \tilde \varphi) \left( \mu_j + \sum_{\ell=1}^d \sum_{i = 1}^{N^{(\ell)}_T} \left( h_{j \ell} \left( T_n^{(j)} - T_i^{(\ell)} \right) + b_{j \ell} \right) \indic{0 < T_n^{(j)} - T_i^{(\ell)} \le A} \right)\right| \\
    &\le N^{(j)}_T \frac{\log 2}{\delta \omega},
  \end{align*}
  by \cref{lem:approximation_bound}.
  
  Combining both bounds:
  \begin{align*}
    L(\hat \theta) - L(\bar \theta)
    &\le 2 \sum_{j=1}^d
    \left\{
      \frac{L_\lambda T^2}{2 M} + \frac{2 N_T G T}{M} + \frac{T \log 2}{\omega} + N^{(j)}_T \frac{\log 2}{\delta \omega}
    \right\} \\
    &= \frac{T}{M} \left( d L_\lambda T + 4 d N_T G \right)
    + \frac{2 \log 2}{\omega} \left( d T + \frac{N_T}{\delta} \right) \\
    &= \frac{T}{M} \left( L_k C d T N_T + 4 (\kappa C + B) d N_T \right)
    + \frac{2 \log 2}{\omega} \left( d T + \frac{N_T}{\delta} \right).
  \end{align*}

\subsubsection*{Proof of \cref{prop:approximation_ls}}
  Let \(j \in \llb 1, d \rrb\) and \(\theta\) be either \(\hat \theta\) or \(\bar \theta\).
  Then, for all \(\omega \ge 1\) and \(t \in [0, T]\),
  \begin{align*}
    0
    \le \lambda^{(j)}(t)
    \le \tilde \lambda^{(j)}(t)
    \le \lambda^{(j)}(t) + \frac{\log 2}{\omega}
    &= \varphi \left( \mu_j + \sum_{\ell=1}^d \sum_{i = 1}^{N^{(\ell)}_T} g_{j \ell} \left(t - T_i^{(\ell)} \right) \right) + \frac{\log 2}{\omega} \\
    &\le \mu_j + \sum_{\ell=1}^d \sum_{i = 1}^{N^{(\ell)}_T} \| g_{j \ell} \|_\infty + \frac{\log 2}{\omega} \\
    &\le B + N_T G + \frac{\log 2}{\omega} \\
    &\le \frac{H}{2},
  \end{align*}
  with \(G = \kappa C + B\) and \(H = 2 ( B + N_T G + \log 2)\).
  Thus, both \(\lambda^{(j)}\) and \(\tilde \lambda^{(j)}\) have values in \([0, H/2]\), interval on which the square function is \(H\)-Lipschitz continuous.
  
  Now, following the proof of \cref{prop:approximation_mle}, we have:
  \begin{align*}
    0
    &\le J(\hat \theta) - J(\bar \theta) \\
    &= J(\hat \theta) - J_{M, \omega}(\hat \theta) + J_{M, \omega}(\hat \theta) - J_{M, \omega}(\bar \theta) + J_{M, \omega}(\bar \theta) - J(\bar \theta) \\
    &\le 2 \max_{\theta \in \{\hat \theta, \bar \theta\}} \left| J(\theta) - J_{M, \omega}(\theta) \right| \\
    &\le 2 \max_{\theta \in \{\hat \theta, \bar \theta\}} \sum_{j=1}^d 
    \left|
      \left[ \int_0^T \lambda^{(j)}(t)^2 \dd t - \frac{T}{M} \sum_{n=1}^M \tilde \lambda^{(j)}(\tau_n)^2 \right]
      - 2\left[ \sum_{n=1}^{N^{(j)}_T} \lambda^{(j)} \left( T_n^{(j)} \right) - \sum_{n=1}^{N^{(j)}_T} \tilde \lambda^{(j)} \left( T_n^{(j)} \right) \right]
    \right|.
  \end{align*}
  Let \(\theta\) be either \(\hat \theta\) or \(\bar \theta\).
  Regarding the first term,
  \begin{align*}
    \left| \int_0^T \lambda^{(j)}(t)^2 \dd t - \frac{T}{M} \sum_{n=1}^M \tilde \lambda^{(j)}(\tau_n)^2 \right|
    &= \left| \sum_{n=1}^M \int_{\tau_n}^{\tau_n + \frac TM} \left( \lambda^{(j)}(t)^2 - \tilde \lambda^{(j)}(\tau_n)^2 \right) \dd t \right| \\
    &\le H \sum_{n=1}^M \int_{\tau_n}^{\tau_n + \frac TM} \left| \lambda^{(j)}(t) - \tilde \lambda^{(j)}(\tau_n) \right| \dd t \\
    &\le H \left( \frac{L_\lambda T^2}{2 M} + \frac{2 N_T G T}{M} + \frac{T \log 2}{\omega} \right),
  \end{align*}
  where \(L_\lambda = L_k N_T C\),
  by a derivation similar to previously.

  In addition, regarding the second term, we immediately have, by \cref{lem:approximation_bound}, that:
  \begin{align*}
   \left| \sum_{n=1}^{N^{(j)}_T} \lambda^{(j)} \left( T_n^{(j)} \right) - \sum_{n=1}^{N^{(j)}_T} \tilde \lambda^{(j)} \left( T_n^{(j)} \right) \right|
    &\le N^{(j)}_T \frac{\log 2}{\omega}.
  \end{align*}
  
  Combining both bounds:
  \begin{align*}
    J(\hat \theta) - J(\bar \theta)
    &\le \frac{H T}{M} \left( L T + 4 N_T d G \right)
    + \frac{2 \log 2}{\omega} \left( H d T + N_T \right).\\
    &= \frac{HT}{M} \left( L_k C d T N_T + 4 (\kappa C + B) d N_T \right)
    + \frac{4 \log 2}{\omega} \left( H d T + N_T \right).
  \end{align*}


  \section{Additional numerical results}
    \label{app:numerics}
\subsection{Synthetic data}
  \cref{img:toy_time} depicts the learning time (average single training time) on a personal computer of the several methods compared on the synthetic data.
  The methods \textbf{Gaussian} and \textbf{Bernstein} are very fast to train because they have only \(d(1 + d U)\) parameters (with \(U=10\) here),
  while \textbf{RKHS} has \(d(1 + N_T + 2d)\) parameters.
  That being said, it is interesting to observe that \textbf{RKHS} is only two times slower than \textbf{Exponential},
  which is a model with \(d (2 + d)\) parameters but that necessitates a specific treatment to compute exactly the compensator
  and that involves a non-convex objective function.
  Let us remark that, on the considered dataset, the most accurate approach is the proposed \textbf{RKHS},
  the training time being the price to pay for the accuracy.
  
  \begin{figure}[ht]
    \center
    \includegraphics[width=.5\textwidth]{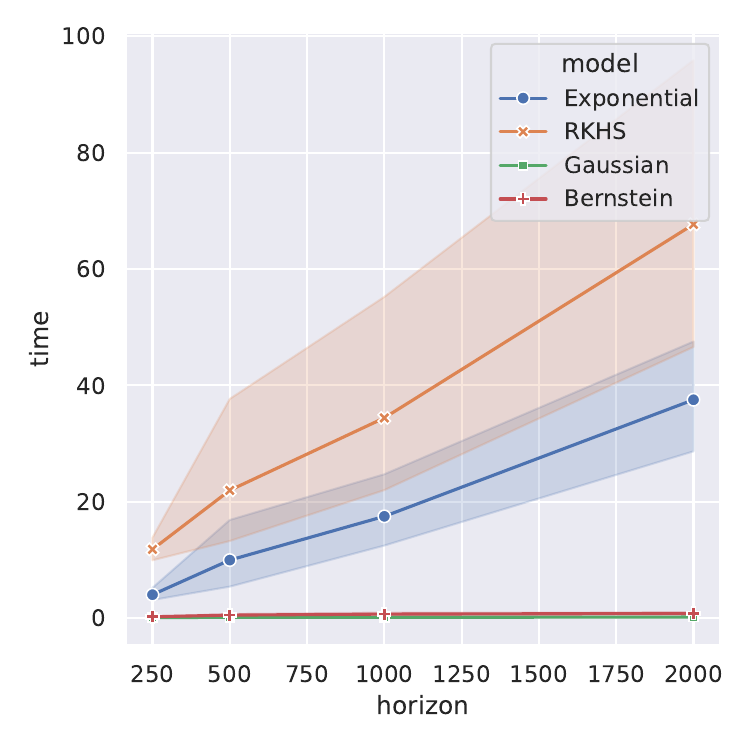}
    \caption{Learning time (in seconds) with respect to the horizon $T$.}
    \label{img:toy_time}
  \end{figure}

\subsection{Neuronal data}
  \cref{img:neuron_all} presents all estimated interaction functions for the several methods considered.
  The proposed \textbf{RKHS}, as well as \textbf{Bernstein} \citep{lemonnier_nonparametric_2014}, recover complex auto-interactions, including the refractory period.
  In return, cross-interactions are estimated close to 0,
  except for minor inhibiting effects by kernels \(1 \gets 3\), \(2 \gets 1\), \(5 \gets 1\), 
  and exciting effects by kernels \(2 \gets 4\), \(3 \gets 4\), \(4 \gets 1\), \(4 \gets 3\).

  \begin{figure*}[ht]
    \center
    \includegraphics[width=\textwidth]{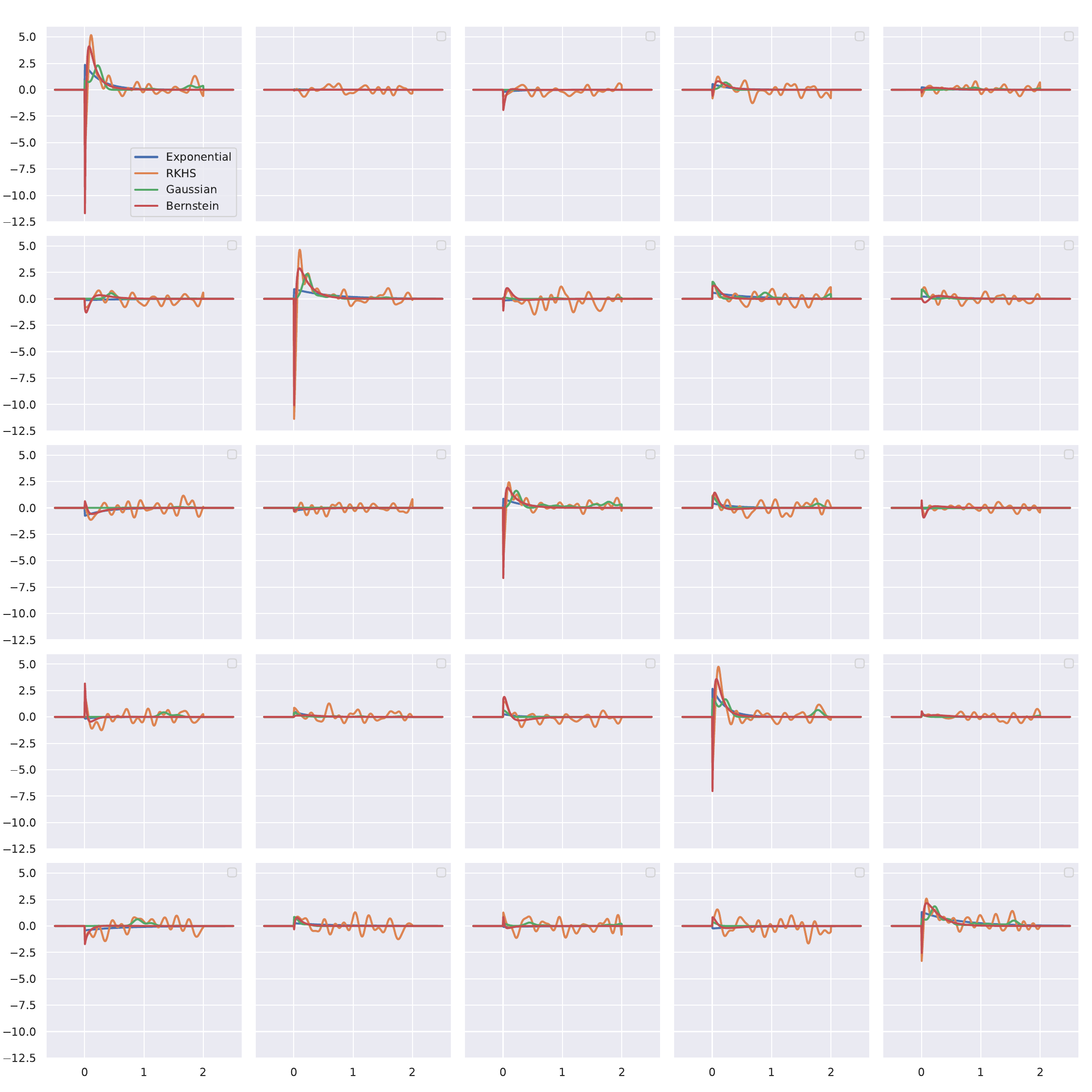}
    \caption{Interaction functions learned on the neuronal subnetwork.}
    \label{img:neuron_all}
  \end{figure*}
  
  \begin{table}[ht]
    \caption{Log-likelihood scores on a test trajectory (the higher, the better).}
    \label{tab:test_loglik}
    \begin{center}
      \begin{tabular}{lc}
        \textsc{Model} & \textsc{Log-Likelihood} \\
        \hline \\
        Exponential & 2152 \\
        RKHS & 2485 \\
        Gaussian & 2178 \\
        Bernstein & 2334 \\
      \end{tabular}
    \end{center}
  \end{table}
  
  As a quantitative assessment,
  \cref{tab:test_loglik} gives the log-likelihood scores computed on a test trajectory obtained by a randomized concatenation of half of the recordings.
  In agreement with \cref{img:neuron_all}, \textbf{RKHS} and \textbf{Bernstein} seem to fit the underlying process better than \textbf{Exponential} and \textbf{Gaussian}.
  Moreover, the highest score is obtained by \textbf{RKHS}, leading to believe that it is more suited than competitors.


\end{document}